\documentclass[10pt,english,reqno]{article}

\usepackage{subcaption}
\usepackage{amssymb}
\usepackage{amsmath}
\usepackage{amsthm}
\usepackage{xcolor}
\usepackage{enumitem}
\usepackage[hidelinks]{hyperref}
\usepackage[textwidth=3cm]{todonotes}
\usepackage{comment}
\usepackage{yhmath}

\usepackage{tikz}
\usetikzlibrary{arrows.meta,positioning}

\DeclareMathOperator*{\argmin}{argmin}

\DeclareMathOperator{\Span}{Span}

\DeclareMathOperator{\mindot}{min.}
\DeclareMathOperator{\Cone}{Cone}
\DeclareMathOperator{\Lasso}{Lasso}
\DeclareMathOperator{\PosLasso}{PosLasso}

\let\P\relax
\DeclareMathOperator{\P}{P}
\DeclareMathAlphabet{\mathbbb}{U}{bbold}{m}{n}

\def\var{{\rm var\,}}
\newcommand{\diff}{\mathrm{d}}

\newcommand{\R}{\mathbb{R}}

\newcommand{\pos}{\text{pos}}
\renewcommand{\neg}{\text{neg}}

\renewcommand{\leq}{\leqslant}
\renewcommand{\geq}{\geqslant}

\renewcommand{\var}{{x}}

\theoremstyle{plain}
\newtheorem{proposition}{Proposition}[section]
\newtheorem{theorem}[proposition]{Theorem}
\newtheorem{lem}[proposition]{Lemma}

\theoremstyle{definition}

\theoremstyle{remark}

\numberwithin{equation}{section}

\author{Raphaël Berthier\footnote{Sorbonne Université, Inria,
Centre Inria de Sorbonne Université,
Paris, France}}

\begin{document}

\title{Diagonal Linear Networks and the~Lasso~Regularization~Path}

\maketitle

\begin{abstract}
Diagonal linear networks are neural networks with linear activation and diagonal weight matrices. Their theoretical interest is that their implicit regularization can be rigorously analyzed: from a small initialization, the training of diagonal linear networks converges to the linear predictor with minimal $1$-norm among minimizers of the training loss. In this paper, we deepen this analysis showing that the full training trajectory of diagonal linear networks is closely related to the lasso regularization path. In this connection, the training time plays the role of an inverse implicit regularization parameter. Both rigorous results and simulations are provided to illustrate this conclusion. Under a monotonicity assumption on the lasso regularization path, the connection is exact while in the general case, we show an approximate connection.
\end{abstract}

\noindent

\newpage

\tableofcontents
\newpage

\section{Introduction}

\paragraph*{Context.}The composition of layers enables a neural network to learn a modular representation of the data. However, the reasons why learning the components of this representation is computationally amenable through gradient descent methods and why it leads to excellent generalization properties are still not fully understood \cite{bartlett2021deep}. 

In order to address this question, machine learning theory has studied the gradient flow training of \emph{linear} networks---where the activation is linear---and \emph{diagonal} linear networks (DLNs)---where, in addition, weight matrices are diagonal. We now provide a brief mathematical introduction to DLNs (see references below for more details).

Consider the minimization of quadratic function: 
\begin{equation*}
	\ell(\var) = \frac{1}{2} \langle \var, M \var \rangle - \langle r, \var \rangle \, , \qquad \var \in \R^d \, . 
\end{equation*}
We assume that this function is convex, i.e.~that $M$ is positive semidefinite, and that it is lower-bounded, i.e.~that $r \in \Span M$.
The minimization of such functions arises, e.g., when solving linear least-squares problems. 

A DLN (with two layers) consists in parametrizing $x = u \circ v$, where $\circ$ denotes the component-wise product of the two vectors $u, v \in \R^d$. We then train the weights $u, v$ by the gradient flow on the loss $\ell(u\circ v)$, potentially regularizing by the $2$-norm of $u,v$ (weight decay):
\begin{align}
	&L(u,v) = \ell(u\circ v) + \frac{\lambda}{2} \left(\Vert u \Vert^2 + \Vert v \Vert^2\right) \, , \qquad \lambda \geq 0 \, , \label{eq:def-L}\\
	&\frac{\diff u}{\diff t} = - \nabla_u L(u,v) \, , \qquad\frac{\diff v}{\diff t} = - \nabla_v L(u,v) \, . \label{eq:flow-uv-new}
\end{align} 
This induces a trajectory for the weights $u(t), v(t)$ as well as an effective linear parameter $x(t) = u(t) \circ v(t)$. 

A first observation is that, under mild assumptions, $x(t)$ converges to a minimizer of the lasso problem \begin{equation*}
	\mindot_{x \in \R^d} \left\{\ell(x) + {\lambda} \Vert x \Vert_1 \right\}\, .
\end{equation*}
See, e.g., \cite{tibshirani2021equivalences}. However, a stronger intriguing phenomenon was observed without weight decay ($\lambda = 0$): when initialized from a suitable infinitesimal initialization~$x(0)$, the training of DLNs is shown to converge to the minimizer of $\ell$ with mininimal $1$-norm \cite{woodworth2020kernel}. As a consequence, the minimizer selected by the training of a DLN benefits from a sparsifying effect, even without any explicit regularization present in the loss function $L$. This phenomenon, called implicit regularization, is beneficial for the generalization properties of the trained network. Implicit regularization has been observed in other neural network structures and suggested to be a key ingredient in the success of neural networks \cite{vardi2023implicit}. See, e.g., \cite{gunasekar2017implicit, li2018algorithmic, chizat2020implicit, li2021implicit} for contributions to the implicit regularization of neural networks and \cite{vaskevicius2019implicit,zhao2022high,woodworth2020kernel,haochen2021shape,li2021implicit,azulay2021implicit,pesme2021implicit,vivien2022label,nacson2022implicit,chou2021more,berthier2023incremental,pesme2023saddle} for more contributions specifically on DLNs. 

\paragraph*{Contributions.} In this article, we show that DLNs enjoy a stronger implicit regularization when early stopped, and connect the training trajectory of DLNs to the lasso regularization path.

For a parameter $\mu \geq 0$, we define the lasso objective as 
\begin{equation}
	\label{eq:def-lasso-objective}
	\Lasso(\var, \mu) = \ell(\var) + \left(\lambda+\frac{1}{\mu}\right) \Vert \var \Vert_1\, .
\end{equation}
This non-standard parametrization of the lasso regularization is convenient to separate the explicit regularization $\lambda \Vert x \Vert_1$ (that stems from weight decay) from the implicit regularization $\frac{1}{\mu} \Vert x \Vert_1$. The results below will clarify why it is convenient to parametrize the implicit regularization by its inverse regularization parameter $\mu$. In all of this paper, the reader can choose to consider the case $\lambda = 0$ for simplicity, as it is the most relevant case for the study of implicit regularization.

The connection between the training trajectory and the lasso regularization path holds in the same infinitesimal initialization limit as the convergence to the minimum $1$-norm solution described above. When $x(0)$ is of magnitude $\varepsilon \ll 1$, the time $t$ is jointly rescaled as $t(s,\varepsilon) = \frac{1}{2} \left(\log \frac{1}{\varepsilon}\right) s$. Consider the average of the trajectory of $x$:
\begin{equation*}
	\overline{x}(t) = \frac{1}{t}\int_0^t x(u) \diff u \, .
\end{equation*}
We then have the following informal result: \emph{for all $s > 0$, as $\varepsilon \to 0$, $\overline{x}(t(s,\varepsilon))$ (approximately) minimizes the lasso objective $\Lasso\left(., {s}\right)$.}

In the informal result above, the parenthesized word ``approximately'' refers to a more technical discussion. Under a certain monotonicity assumption on the lasso regularization path, we show that $\Lasso\left(\overline{x}(t(s,\varepsilon)), s\right)$ converges to the minimum of $\Lasso\left(., s\right)$ as $\varepsilon \to 0$. However, when the lasso regularization path is not monotone, we only show that $\overline{x}(t(s,\varepsilon))$ is an approximate minimizer of $\Lasso\left(., {s}\right)$, even in the limit $\varepsilon \to 0$, where the suboptimality gap is controled by the deviation from the monotonicity assumption. This distinction is also apparent in simulations. 

Overall, this result shows that early stopping the training of DLNs can offer a tradeoff between sparsity and data fitting. An earlier stopping time leads to a sparser linear model. This deepens the connection between DLNs and sparse regression. 

Our results also cover the case $x = u \circ u$, that is less motivated by the theoretical study of neural networks, but that is algebraically more elegant. As a consequence, we provide the proofs in the $u \circ u$ case first, and then deduce the proofs in the $u \circ v$ by a reduction to the $u \circ u$ case. This systematic reduction strategy might be of independent interest for the study of DLNs. 

\paragraph*{Outline.} Sec.~\ref{sec:uv} is the one of interest for most readers. It contains the results in the case $x = u \circ v$ and their discussions. Sec.~\ref{sec:monotone-uv} provides the exact connection under the monotonicity assumption on the lasso regularization path; Sec.~\ref{sec:simulations} presents simulations that illustrate the results; and Sec.~\ref{sec:nonmonotone-uv} covers the general case, beyond the monotonicity assumption. Sec.~\ref{sec:conclusion} concludes this section by discussing in more detail the related work and a speculative conclusion for the theory of neural networks. 

Sec.~\ref{sec:ucarre} contains the parallel results in the case $x = u \circ u$. Sec.~\ref{sec:proofs-ucarre} contains the proofs of the results in the $u \circ u$ case. This section contains the bulk of the technical effort of this paper, as the proofs in the $u \circ v$ case, provided in Sec.~\ref{sec:uv-proof}, are obtained by a systematic reduction to the $u \circ u$ case.

\paragraph*{Notations.}
	We denote $\langle . , . \rangle$ the canonical inner product on $\R^d$ and $\Vert . \Vert$ the associated norm. We note $\Vert . \Vert_1$ the $1$-norm on vectors: $\Vert x \Vert_1 = \sum_{i=1}^d |x_i|$.
	
	We use the notation $\var \geq 0$ to denote that $\var$ has all coordinates nonnegative. If $x$ is a real number, we denote $x_+ = \max(x,0)$ its positive part and $x_- = \max(-x, 0)$ its negative part.
	
	When $\varphi:\R\to\R$ and $x \in \R^d$, we denote $\varphi(x) = (\varphi(x_1), \dots, \varphi(x_d))$ the component-wise application of $\varphi$ to $x$. This notation is used, for instance, for $\varphi = \log$, for $\varphi(x) = x^2$, for $\varphi(x) = x_+$ or for $\varphi(x) = x_-$. We thus denote $x^2 = x \circ x$.

	We sometimes write explicitly the dependence of some positive constants on some parameters, for instance $C_i = C_i(M, r, \alpha)$. Without any index, the notation $C(\dots)$ indicates a positive constant that depends on the parameters between the parentheses. This constant can change from one occurrence to the next. 

	We denote $\mathbbb{1} = (1, \dots, 1)$ the vector of ones, whose dimension is implicit from the context.

	When $I$ is a subset of $\{1, \dots, d\}$, we denote $A_I$ the submatrix of $A$ where the columns whose index are in $I$ are selected, and $A_{I,I}$ the submatrix of $A$ where both the rows and the columns whose index are in $I$ are selected.

	When $f$ is a function of $z$, we denote $\mindot_{z} f(z)$ the minimization problem of~$f$, and $\min_z f(z)$ its minimal value. In other words, $\mindot$ stands for ``minimize'' and $\min$ stands for ``minimum''.

\section{The \texorpdfstring{$u \circ v$}{uv} case --- statement of the results}
\label{sec:uv}

\paragraph*{Setting.} We study the solution of Eqs.~\eqref{eq:def-L}, \eqref{eq:flow-uv-new} in the limit of small initialization. More precisely, we denote $u^\varepsilon(t)$ and $v^\varepsilon(t)$ the solutions of Eqs.~\eqref{eq:flow-uv-new} with initial conditions $u^\varepsilon(0) = \sqrt{\varepsilon} \beta$ and $v^\varepsilon(0) = \sqrt{\varepsilon} \gamma$, where $\beta$ and $\gamma$ are fixed vectors in $\R^d$ and $\varepsilon$ is a small parameter. We denote $\var^\varepsilon(t) = u^\varepsilon(t) \circ v^\varepsilon(t)$. The scalings in $\varepsilon$ are chosen so that $\var^\varepsilon(0)$ is of order $\varepsilon$. For non-degeneracy reasons, we assume that for all $i \in \{1, \dots, d\}$, $\beta_i \neq \pm \gamma_i$.

We also define $\overline{\var}^\varepsilon(t) = \frac{1}{t}\int_0^t \var^\varepsilon(u) \diff u$ the average of the trajectory of $\var^\varepsilon$. The goal of our analyses is to connect the average of the trajectory to the lasso optimization problem
\begin{equation}
	\label{eq:lasso}
	\mindot_{x \in \R^d} \Lasso(x, \mu) \, .  
\end{equation}
We denote $\Lasso_*(\mu) = \min_{x\in \R^d} \Lasso(x, \mu)$ the minimum of this optimization problem.

Define a rescaled time 
\begin{align}
	\label{eq:rescaled-time-uv}
	&s = \frac{2}{\log \frac{1}{\varepsilon}}t \, .
\end{align}

Below, we frequently abuse notations and use the same notation for rescaled functions of time. For instance, $\overline{x}^\varepsilon(s)$ denotes $\overline{x}^\varepsilon(t)$ where $t = t(\varepsilon) = \frac{s}{2}\log \frac{1}{\varepsilon}$.

\subsection{Connection to the lasso under a monotonicity assumption}
\label{sec:monotone-uv}

\begin{theorem}
	\label{thm:monotone-uv}
	For all $\mu >0$, let ${x}(\mu)$ denote a minimizer of the lasso
	\begin{equation*}
		\underset{{x} \in \R^d}{\mindot} \Lasso\left({x}, {\mu}\right) \, .
	\end{equation*}
	Assume that $\mu > 0 \mapsto \mu{x}(\mu)$ is coordinate-wise monotone. Then, for all $s>0$, 
	\begin{equation*}
		\Lasso \left(\overline{\var}^\varepsilon(s) , {s}\right) \xrightarrow[\varepsilon \to 0]{} \Lasso_*\left({s}\right) \, .
	\end{equation*}
\end{theorem}

We insist that in the above result, the rescaled time $s = \frac{2}{\log 1/\varepsilon}t$ is kept fixed and $t = t(\varepsilon) = \frac{s}{2}\log \frac{1}{\varepsilon}$ diverges as $\varepsilon \to 0$. This is necessary to obtain a non-trivial limit for $\overline{x}^\varepsilon$. 

This result identifies the rescaled time $s$ with the inverse implicit regularization parameter~$\mu$ of the lasso. In the limit $\varepsilon \to 0$, a single training trajectory of the DLN computes the full lasso regularization path. Different stopping times correspond to different regularization levels; the earlier the stopping, the more regularized the solution.

For instance, for $s$ small enough, the regularization level is large enough so that the unique lasso minimizer is $0 \in \R^d$. This is consistent with the fact that the DLN dynamics \eqref{eq:flow-uv-new} are initialized $\sqrt{\varepsilon}$-close to the fixed point $(u,v)= (0,0)$, thus they will escape this fixed point in a time proportional to $\log \frac{1}{\varepsilon}$. As a consequence, for $s$ small enough, $x = u\circ v$ is at $0$ (in the limit $\varepsilon \to 0$).

Further, as $s \to \infty$, we obtain that the regularization level is infinitesimal. In the case $\lambda = 0$ (no weight decay), this is consistent with the fact that the DLN dynamics converge to a minimizer of $\ell$ with minimal $1$-norm \cite{woodworth2020kernel}. 

However, Thm.~\ref{thm:monotone-uv} provides an interpolation between these two asymptotics. For finite $s$, $\overline{x}^\varepsilon(s)$ minimizes $\Lasso(.,s)$ in the limit $\varepsilon \to 0$. 

We insist on the fact that $\Lasso(.,s)$ does not have a unique minimizer a priori, thus one can not define \emph{the} lasso minimizer or \emph{the} lasso regularization path. Our claim is that, if there is \emph{a} lasso regularization path $x(\mu)$ such that $\mu \mapsto \mu x(\mu)$ is monotone, then $\overline{x}^\varepsilon(s)$ minimizes $\Lasso(.,s)$ in the limit $\varepsilon \to 0$. In particular, we do not claim that $\overline{x}^\varepsilon(s) \xrightarrow[\varepsilon \to 0]{} x(s)$. To the best of our understanding, the limit of $\overline{x}^\varepsilon(s)$ depends more subtely on the initialization vectors $\beta$ and $\gamma$; this will not be discussed further in this paper.

A sufficient condition for $\mu \mapsto \mu {x}(\mu)$ to be monotone is to have the lasso regularization path $\mu \mapsto {x}(\mu)$ itself monotone (as $\lim_{\mu\to0^+}{x}(\mu) = 0$). Under the latter condition, the lasso regularization path has already been shown to be equivalent to computationally cheaper algorithms such as the least angle regression algorithm~\cite{efron2004least}, incremental forward stagewise regression~\cite{hastie2007forward} or boosting~\cite{rosset2004boosting}. Moreover, it is also argued that even when $\mu \mapsto x(\mu)$ is not monotone, the solutions of these different algorithms might not be very far, see for example Figure $1$ in \cite{efron2004least}. A sufficient condition for the lasso regularization path to be monotone, derived in the latter article, is to have $S (M_{II})^{-1} S \mathbbb{1} \geq 0$ for all subsets $I \subset \{1, \dots, d \}$ and for all diagonal sign matrices $S \in \R^{\vert I \vert \times \vert I \vert}$ (see also \cite{hastie2007forward}, Sec.~6, for a succinct statement).

The monotonicity of $\mu \mapsto \mu x(\mu)$ specifically has been studied, to the best of our knowledge, only in the case of the positive lasso, which corresponds to the parametrization $x = u^2$, see Sec.~\ref{sec:monotone-u2}. 

\subsection{Simulations}
\label{sec:simulations}

In order to illustrate the results of this paper and discuss the importance of the monotonicity assumption, we provide some simulations. The code is available on github \cite{berthier_dln_lasso}.

We randomly generate problem instances in dimension $d = 4$ as follows. We consider the quadratic function $\ell(x) = \frac{1}{2} \Vert Xx - y \Vert^2$, where $n = 3$, $X \in \R^{n \times d}$ and $y \in \R^n$ are independent and have i.i.d.~standard normal entries. We do not put any explicit regularization: $\lambda = 0$.

In this case, the lasso regularization path is almost surely unique \cite{tibshirani2013lasso}. We observe empirically that $\mu \mapsto \mu x(\mu)$ is monotone with probability $0.76$, over $1000$ random instances. Thus Theorem~\ref{thm:monotone-uv} applies in a majority of cases.

In all simulations, we take $\varepsilon = 10^{-5}$, $\beta = \mathbbb{1}$ and $\gamma = 0$.

Figure~\ref{fig:simulations} compares the average trajectory $\overline{x}^\varepsilon(s)$ of DLNs with the lasso regularization path $x(\mu)$ on problems generated randomly as above. When identifying $s = \mu$, the trajectory $\overline{x}^\varepsilon(s)$ is qualitatively similar to $x(s)$ in all instances, and the suboptimality gap $\left(\Lasso\left(\overline{x}^\varepsilon(s), s\right) - \Lasso_*\left(s\right)\right)/\Lasso_*\left(s\right)$ remains relatively small (below $0.04$ on the provided instances). However, the comparison is much tighter in the monotone case than in the nonmonotone case; in particular, the suboptimality gap is an order of magnitude smaller for monotone instances.

\begin{figure}
	\begin{subfigure}{\linewidth}
		\includegraphics[width=0.48\linewidth]{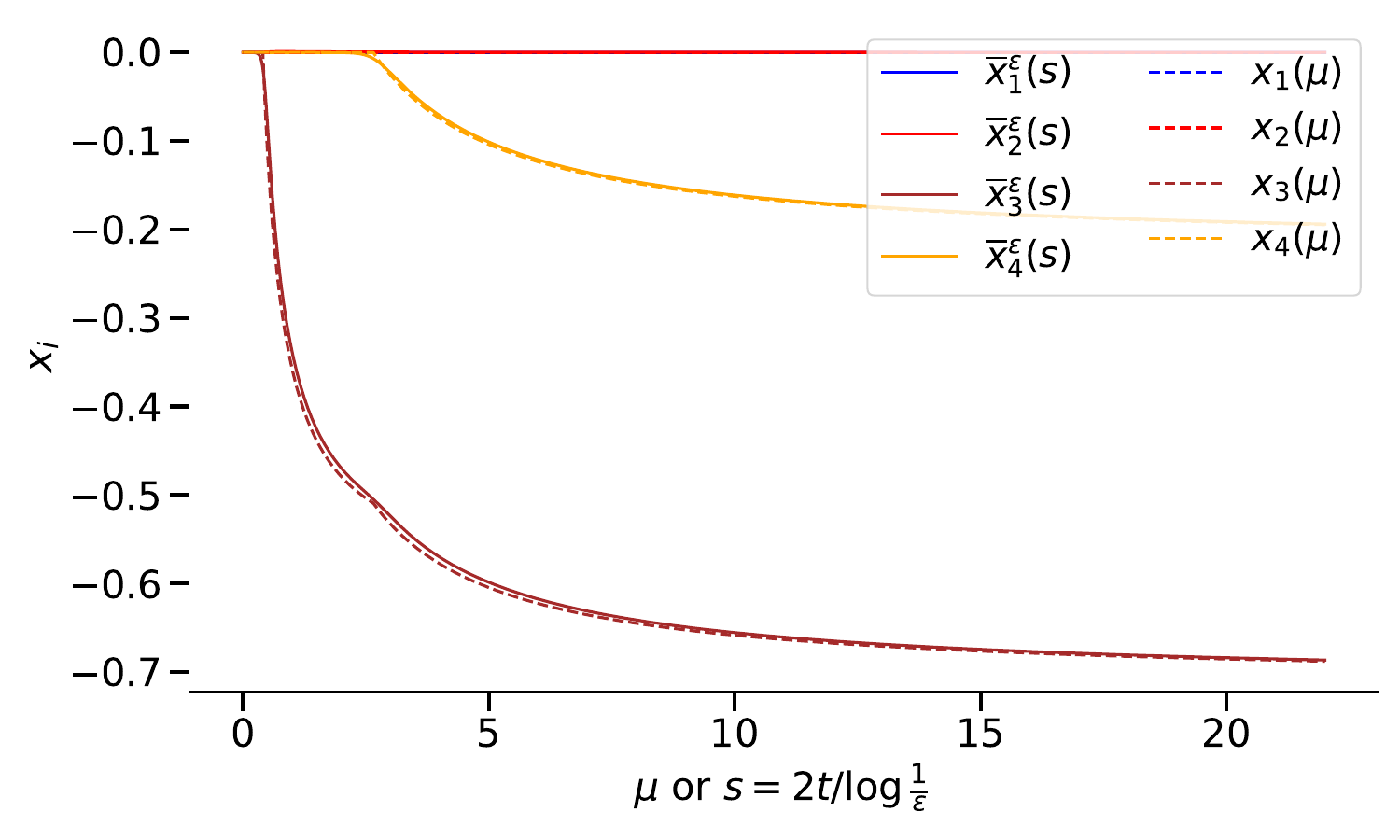}
		\includegraphics[width=0.48\linewidth]{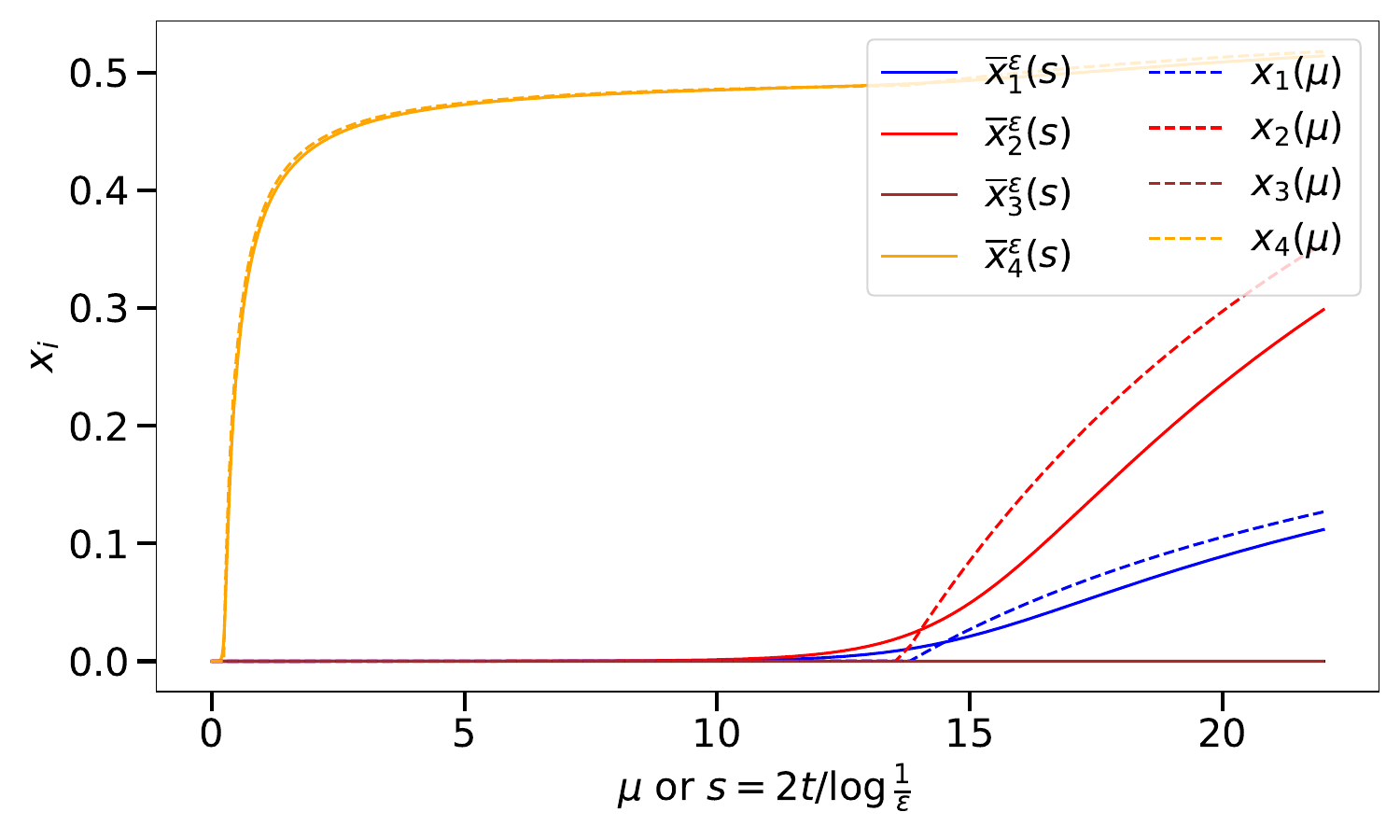}
		\includegraphics[width=0.48\linewidth]{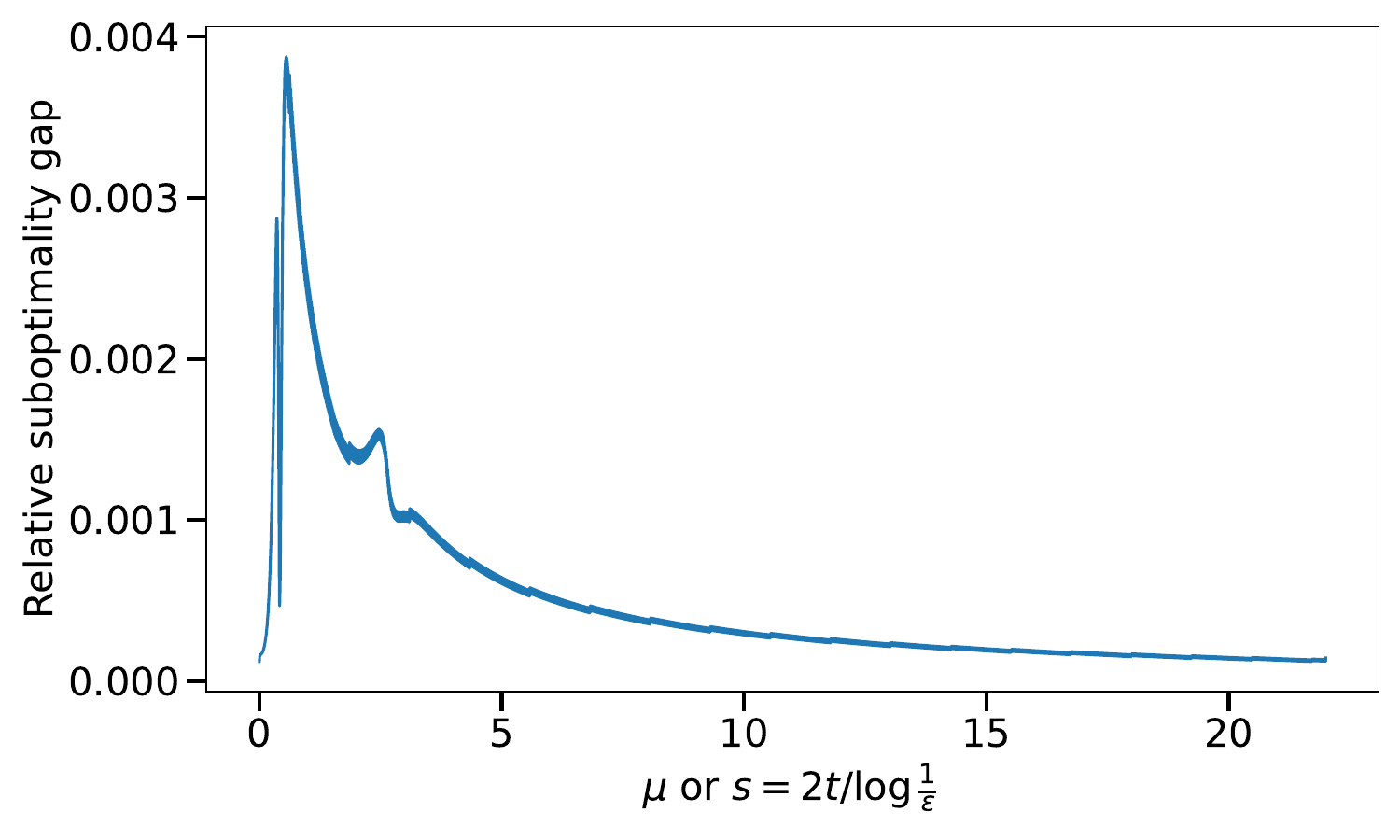}
		\includegraphics[width=0.48\linewidth]{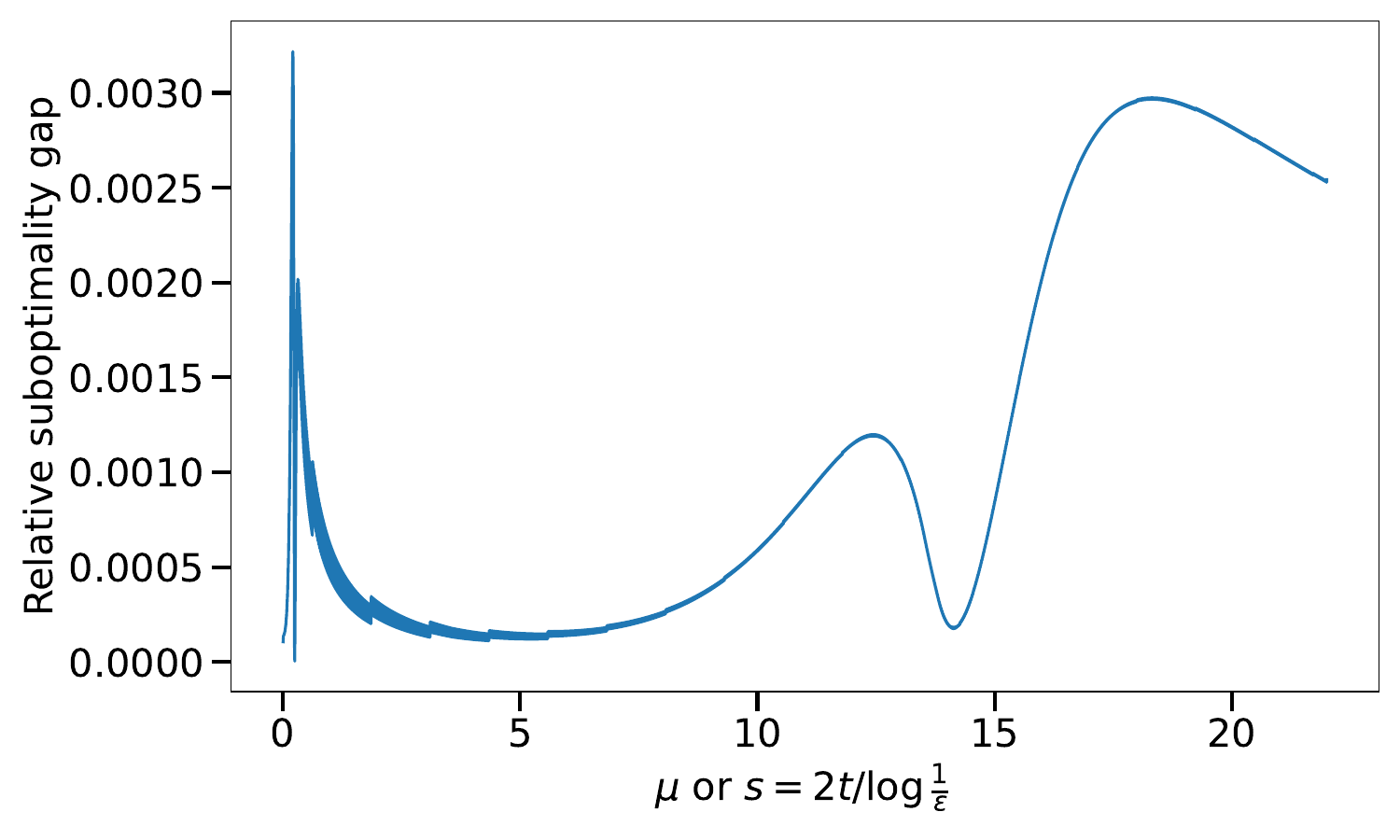}
		\vspace*{-4mm}
		\caption{Monotone instances}
		\vspace*{5mm}
		\label{subfig:monotone}
	\end{subfigure}
	\begin{subfigure}{\linewidth}
		\includegraphics[width=0.48\linewidth]{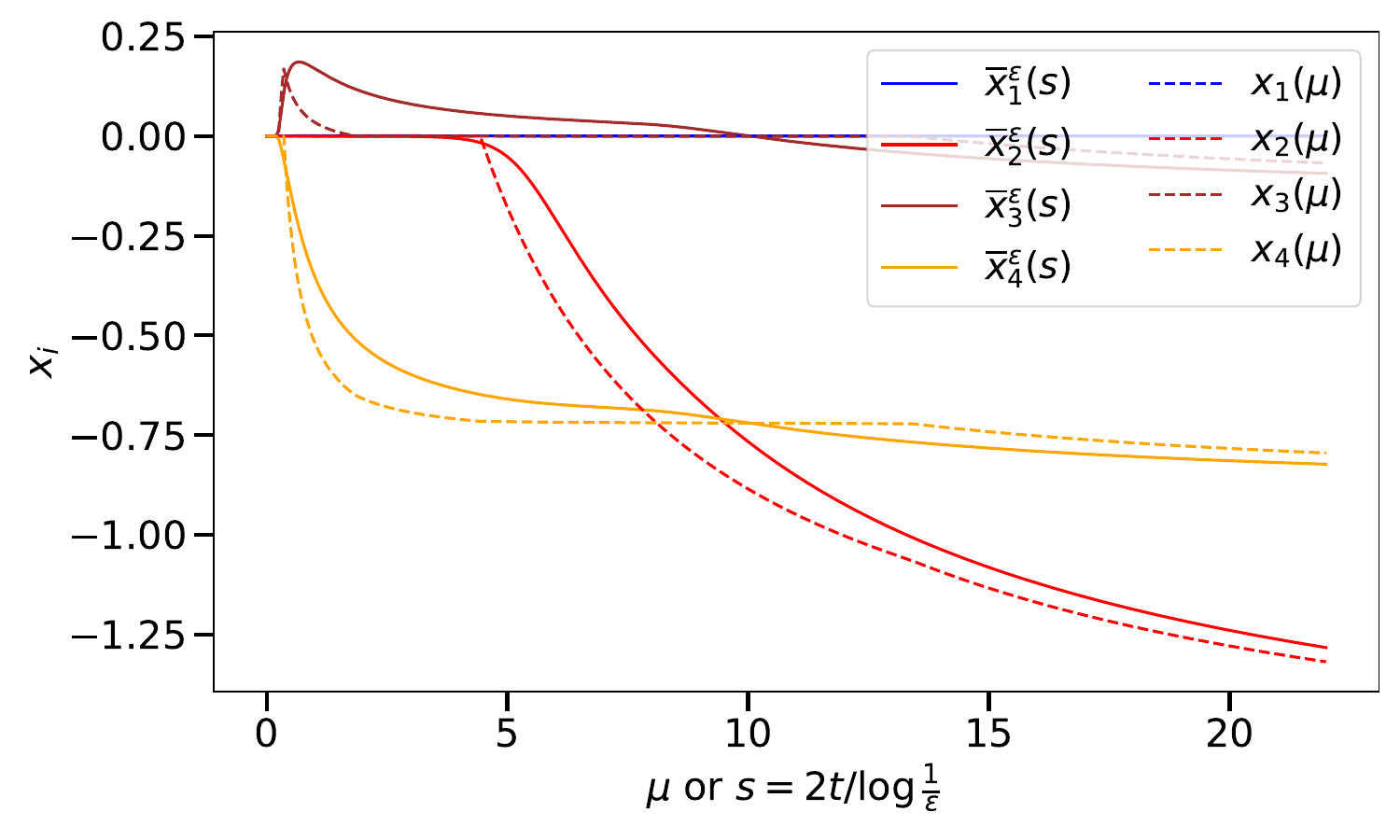}
		\includegraphics[width=0.48\linewidth]{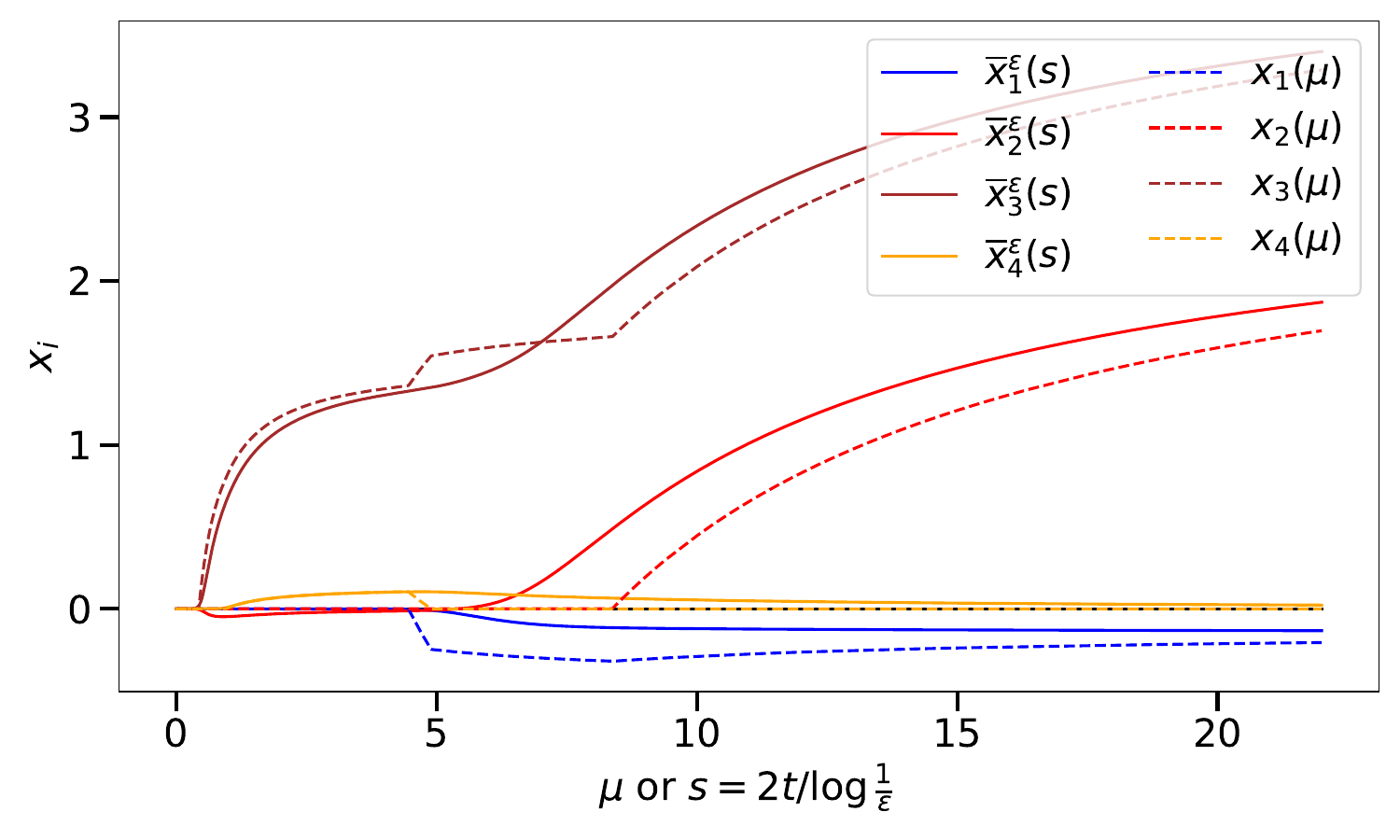}
		\includegraphics[width=0.48\linewidth]{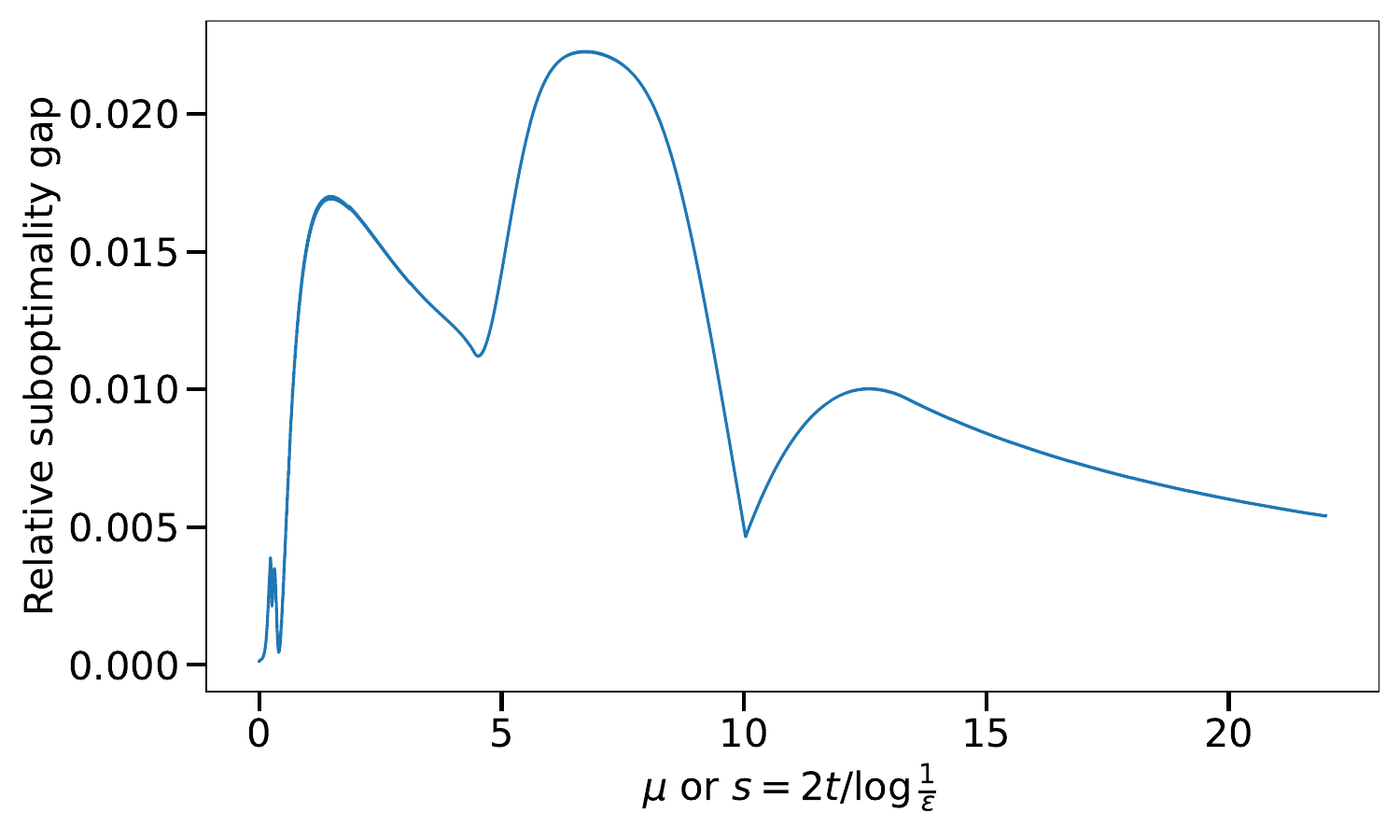}
		\includegraphics[width=0.48\linewidth]{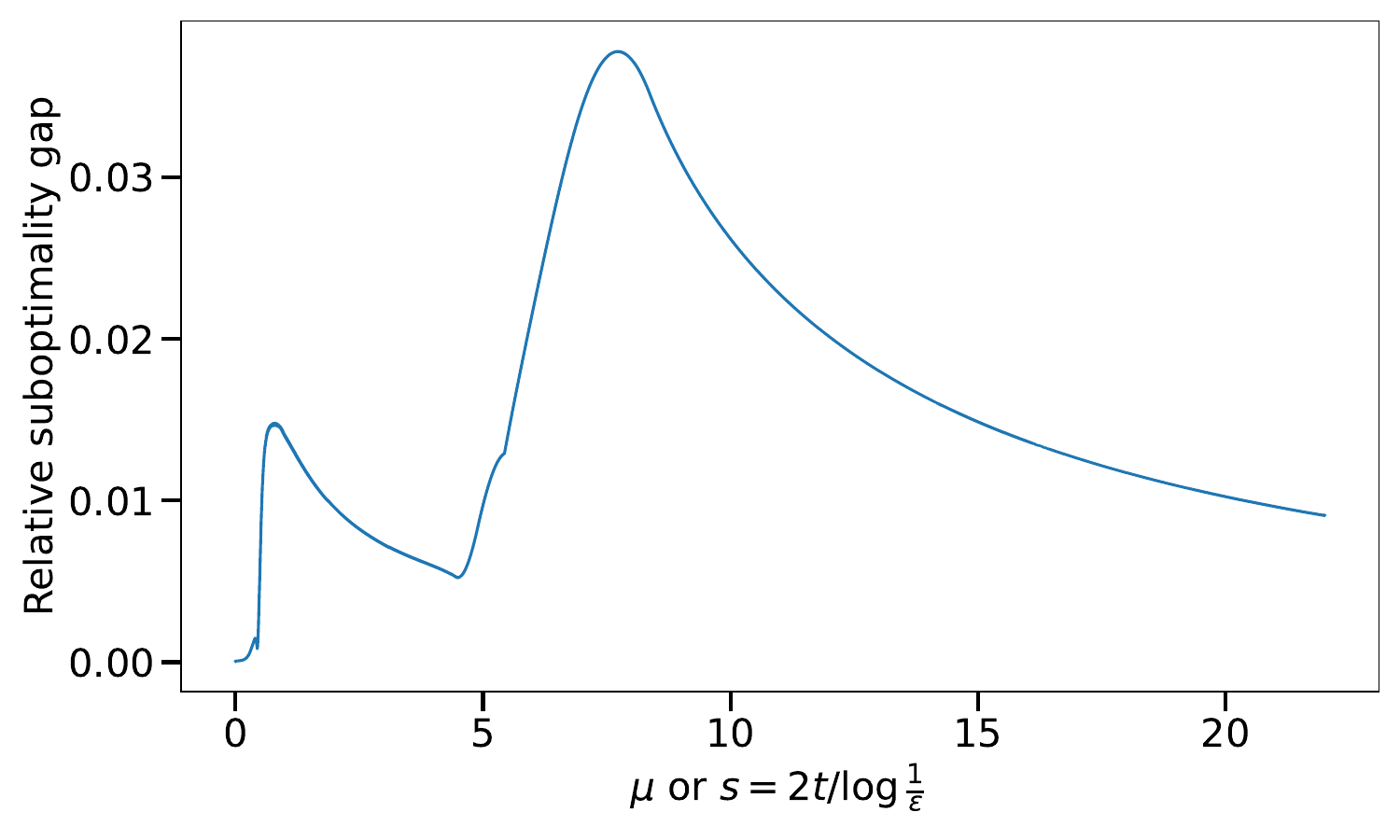}
		\vspace*{-4mm}
		\caption{Nonmonotone instances}
	\end{subfigure}
	\caption{Comparison between the average trajectory $\overline{x}^\varepsilon(s)$ of DLNs and the lasso regularization path $x(\mu)$. In the subfigure (a), problems instances are generated randomly conditionally on the monotonicity of $\mu \mapsto \mu x(\mu)$. Conversely, in the subfigure (b), problems instances are generated randomly conditionally on the non-monotonicity of $\mu \mapsto \mu x(\mu)$. We provide two instances of each case, one in each column. In each instance, we plot the coordinates $\overline{x}^\varepsilon_i(s)$ along with $x_i(\mu)$, and the suboptimality gap $\left(\Lasso\left(\overline{x}^\varepsilon(s), s\right) - \Lasso_*\left(s\right)\right)/\Lasso_*\left(s\right)$. Simulation details are provided in Sec.~\ref{sec:simulations}.}
	\label{fig:simulations}
\end{figure}

\subsection{Approximate connection to the lasso in the general case}
\label{sec:nonmonotone-uv}

We now provide an approximate optimality result in the nonmonotone case that comforts the simulations. 

\begin{theorem}
	\label{thm:nonmonotone-uv}
	For all $\mu > 0$, let ${x}(\mu)$ denote a minimizer of the lasso
	\begin{equation*}
		\underset{{x} \in \R^d}{\mindot} \Lasso\left({x}, {\mu}\right) \, .
	\end{equation*}
	Assume that $\mu > 0 \mapsto {x}(\mu)$ is absolutely continuous on compact subsets of $(0, +\infty)$. Define $z(\mu) = \mu {x}(\mu)$ and
	\begin{equation}
		\label{eq:z-downarrow-uv}
		z^{\downarrow}(\mu) = \sum_{i=1}^{d} \int_{0}^{\mu} (1+u)\left[\left(\frac{\diff \left(z_i(u)_+\right)}{\diff u} \right)_- + \left(\frac{\diff \left(z_i(u)_-\right)}{\diff u} \right)_-\right] \diff u \, . 
	\end{equation}
	Heuristically, $z^{\downarrow}(\mu)$ quantifies the deviation from monotonicity of $z(\mu)$. 

	Then, there exists a constant $C = C(d,M,r,\lambda, \beta, \gamma) > 0$ such that, for all $s>0$,
	\begin{align}
		&\limsup_{\varepsilon \to 0} \Lasso \left(\overline{\var}^\varepsilon(s) , s\right) 
		\leq \Lasso_*\left(s\right) + C \eta(\lambda, s, z^\downarrow(s)) \, , \nonumber\\
		&\text{where } \eta(\lambda, s, z^\downarrow) = (1+\lambda s) \left[\frac{(z^\downarrow)^{1/2}}{s} + \frac{z^\downarrow}{s^2} \right] \, . \label{eq:def-eta}
	\end{align}
\end{theorem}

Note that $\eta(\lambda, s, 0) = 0$; this enables to derive Thm.~\ref{thm:monotone-uv} from Thm.~\ref{thm:nonmonotone-uv}. The above theorem predicts that the suboptimality gap spikes when the lasso regularization path is not monotone (when $z^\downarrow(s)$ is large) and then decays afterwards (as $s$ increases). This is consistent with the observations in Figure~\ref{fig:simulations}.

\subsection{Related work and final remarks}
\label{sec:conclusion}

\paragraph*{Additional related work.}
In contrast with many works on diagonal linear networks, this paper does not assume any statistical model for the data that generated the loss $\ell$, and focuses only on the optimization analysis. In particular, our results could cover noisy or noiseless data, and general design matrices. This is in contrast with many works that require a restricted isometry property, or a coherence assumption on the design matrix \cite{vaskevicius2019implicit,zhao2022high, li2021implicit}. 

Under a statistical linear model with additive noise, early stopping is used in DLNs without weight decay to achieve optimal rates \cite{vaskevicius2019implicit,zhao2022high, li2021implicit}. Similarly, in the same setting, the regularization parameter of the lasso is usually chosen proportional to the magnitude of the additive noise, see \cite{hastie2015statistical}, Chapter 11. Our work sheds light on this correspondence: a non-zero regularization parameter corresponds to a finite stopping time.

More specifically, some earlier works have noted the similarity between the trajectory of DLNs and the lasso regularization path \cite{vaskevicius2019implicit, berthier2023incremental, pesme2023saddle}. Vaskevicius et al.~noted that $t / \log \frac{1}{\varepsilon}$ plays the role of an inverse regularization parameter in their statistical guarantee~\cite{vaskevicius2019implicit} (see discussion below their Thm.~1). However, they note that ``the gradient descent optimization path [of DLNs] [...] exhibits qualitative and quantitative differences from the lasso regularization path.'' Our contribution argues that the qualitative differences are actually due to the fact that the lasso regularization path corresponds to the \emph{averaged} gradient descent optimization path. Finally, the author proved Thm.~\ref{thm:monotone-u2}, the equivalent of Thm.~\ref{thm:monotone-uv} in the case $x = u \circ u$, but only when $M$ is a matrix with non-positive off-diagonal entries~\cite{berthier2023incremental}. (In this case, the monotonicity assumption is guaranteed.)

\paragraph*{Final remarks.} This work strengthens our understanding of the implicit regularization of neural networks. Not only the asymptotic convergence point is implicitely regularized, but also the full trajectory of the neural network. The earlier the neural network is stopped, the more regularized the solution is. 

Our proof technique provides a way to analyze the implicit regularization in a \emph{dynamical} way, that is, as time evolves. We hope that this new perspective will help to tackle a major difficulty of the implicit regularization line of research: we do not know how to describe the implicit regularization of more complex neural networks, beyond DLNs. For instance, the implicit regularization of matrix factorization problems can not be described by the nuclear norm, as one would expect from a natural generalization of the results on DLNs \cite{arora2019implicit,li2020towards}. We hope that this discrepancy can be understood through a dynamical study of the deviation between matrix factorization dynamics and nuclear norm minimizers.

\section{The \texorpdfstring{$u \circ u$}{u squared} case --- statement of the results}
\label{sec:ucarre}

\paragraph*{Setting.}In this section, we state an alternative version of the results of this paper, in the case where the regressor $\var$ is parametrized as $\var = u^2 = u \circ u$ instead of $\var = u \circ v$. It could be seen as a weight-tied version of DLNs where we impose the same weights in the two layers: $u=v$. Actually, this weight-tied model of neural networks is not of direct interest, and it has the advantage of constraining the sign of $\var$. However, it enables a more elegant theory. Moreover, as we discuss in Section~\ref{sec:uv-proof}, the case $\var = u \circ v$ can be reduced to the case $\var = u^2$, thus in fact the bulk of the theory is contained in this case, and the case $\var = u \circ v$ follows as corollaries.

Taking $u=v$ in Eqs.~\eqref{eq:def-L}, \eqref{eq:flow-uv-new}, we obtain the dynamics in the $x = u \circ u$ case:
\begin{align}
	&L(u) = \ell(u^2) + \lambda \Vert u \Vert^2 \, , \qquad \lambda \geq 0 \, , \label{eq:def-L-u2} \\
	&\frac{\diff u}{\diff t} = - \nabla_u L(u) \, . \label{eq:flow-u2-new}
\end{align}
In this case, the gradient flow dynamics can be written a closed-form differential equation in $\var = u^2$:
\begin{align}
	\label{eq:flow-u2-theta}
	\frac{\diff \var}{\diff t} = - 4 \var \circ \nabla \ell(\var) - 4 \lambda \var \, .
\end{align}
In this section, we denote $u^\varepsilon(t)$ the solution of Eq.~\eqref{eq:flow-u2-new} with initial condition $u^\varepsilon(0) = \sqrt{\varepsilon} \alpha$ where $\alpha$ is a fixed vector in $\R^d$ and $\varepsilon$ is a small parameter. We denote $\var^\varepsilon(t) = u^\varepsilon(t)^2$. We assume that all coordinates $\alpha_i$ of $\alpha$ are nonzero: if $\alpha_i = 0$, by Eq.~\eqref{eq:flow-u2-theta}, the coordinate of $\var^\varepsilon_i(t)$ remains zero for all $t \geq 0$, thus this case is degenerate. 

Again, we define $\overline{\var}^\varepsilon(t) = \frac{1}{t}\int_0^t \var^\varepsilon(u) \diff u$ the average of the trajectory of $\var$. The goal of our analyses is to connect the average of the trajectory to the positive lasso optimization problem 
\begin{equation}
	\label{eq:positive-lasso}
	\underset{x \geq 0}{\mindot} \Lasso\left(x, \mu\right) \, ,
\end{equation}
where the lasso objective is defined in Eq.~\eqref{eq:def-lasso-objective}. 

We denote $\PosLasso_*(\mu) = \min_{x \geq 0} \Lasso(x, \mu)$ the minimum of this optimization problem. 

Define a rescaled time 
\begin{align}
	\label{eq:rescaled-time-u2}
	&s = \frac{4}{\log \frac{1}{\varepsilon}}t \, .
\end{align}

\subsection{Connection to the positive lasso under a monotonicity assumption}
\label{sec:monotone-u2}

\begin{theorem}
	\label{thm:monotone-u2}
	For all $\mu >0$, let ${x}(\mu)$ denote a minimizer of the positive lasso
	\begin{equation*}
		\underset{{x} \geq 0}{\mindot} \Lasso\left({x}, {\mu}\right) \, .
	\end{equation*}
	Assume that $\mu > 0 \mapsto \mu{x}(\mu)$ is coordinate-wise nondecreasing. Then, for all $s>0$, 
	\begin{equation*}
		\Lasso \left(\overline{\var}^\varepsilon(s) , {s}\right) \xrightarrow[\varepsilon \to 0]{} \PosLasso_*\left({s}\right) \, .
	\end{equation*}
\end{theorem}

The monotonicity of $\mu \mapsto \mu x(\mu)$ translates into the monotonicity of the solution of the parametric linear complementarity problem \eqref{eq:parametric-LCP}, introduced in Sec.~\ref{sec:param-LCP}. This monotonicity has first been studied in a different context, the analysis of elastoplastic structures~\cite{cottle1973parametric}. Cottle proposed an algorithm to test whether the solution of a given problem is monotone, and derived necessary and sufficient conditions under which monotonicity holds for all possible values of~$r$~\cite{cottle1972monotone}. See also~\cite{megiddo1977monotonicity} when $M$ is non-invertible. A sufficient condition for monotonicity is when $M$ is a positive definite matrix with non-positive off-diagonal entries. The author has proved Thm.~\ref{thm:monotone-u2} under this more restrictive assumption~\cite{berthier2023incremental}.

\subsection{Approximate connection to the positive lasso in the general case}
\label{sec:nonmonotone-u2}

We now turn to the general case.

\begin{theorem}
	\label{thm:nonmonotone-u2}
	For all $\mu > 0$, let ${x}(\mu)$ denote a minimizer of the positive lasso
	\begin{equation*}
		\underset{{x} \geq 0}{\mindot} \Lasso\left({x}, {\mu}\right) \, .
	\end{equation*}
	Assume that $\mu > 0 \mapsto {x}(\mu)$ is absolutely continuous on compact subsets of $(0, +\infty)$. Define $z(\mu) = \mu {x}(\mu)$ and
	\begin{equation}
		\label{eq:z-downarrow-u2}
		z^{\downarrow}(\mu) = \sum_{i=1}^{d} \int_{0}^{\mu} (1+u) \left(\frac{\diff z_i(u)}{\diff u} \right)_- \diff u \, . 
	\end{equation}
	Heuristically, $z^{\downarrow}(\mu)$ quantifies the deviation from monotonicity of $z(\mu)$. 

	Then, there exists a constant $C = C(d,M,r,\lambda,\alpha) > 0$ such that, for all $s>0$,
	\begin{align*}
		&\limsup_{\varepsilon \to 0} \Lasso \left(\overline{\var}^\varepsilon(s) , {s}\right) 
		\leq \PosLasso_*\left({s}\right) + C \eta(\lambda, s, z^\downarrow(s)) \, , 
	\end{align*}
	where $\eta(\lambda, s, z^\downarrow)$ is defined in Eq.~\eqref{eq:def-eta}.
\end{theorem}

\section{The \texorpdfstring{$u \circ u$}{u squared} case --- proofs of the results}
\label{sec:proofs-ucarre}

This section is organized as follows. In Sec.~\ref{sec:sketch}, we provide a sketch of the proof of Thm.~\ref{thm:monotone-u2}. The following four subsections gather some elements to prepare the detailed proofs. In Sec.~\ref{sec:mirror-flow}, we recall the mirror flow perspective on the DLN dynamics \cite{li2022implicit}. In Sec.~\ref{sec:uniform-bound}, we prove a preliminary bound on the DLN dynamics, that is uniform in time $t$ and for $\varepsilon$ small enough. Its proof is independent from the rest of the proofs, and can be skipped in a first reading. Sec.~\ref{subsubsec:LCP} introduces linear complementary problems. Sec.~\ref{sec:param-LCP} extends this introduction to \emph{parametric} linear complementarity problems, where a parameter $\mu$ varies in the problem. Finally, in Sec.~\ref{sec:proof-nonmonotone-u2}, we prove Thm.~\ref{thm:nonmonotone-u2} and in Sec.~\ref{sec:proof-monotone-u2}, we deduce the proof of Thm.~\ref{thm:monotone-u2}. 

\paragraph*{Additional notations for the proofs.}  
When applied to a $d \times d$ matrix, $\Vert . \Vert$ denotes the operator norm associated to the norm $\Vert . \Vert$ on $\R^d$. When $A$ is a positive semidefinite matrix, we denote $\langle . , . \rangle_A = \langle ., A . \rangle$ the semidefinite inner product associated to $A$ and $\Vert . \Vert_A$ the associated seminorm.

When $A \in \R^{d \times k}$ is a matrix, we denote $A^\dagger \in \R^{k \times d}$ its Moore-Penrose pseudo-inverse.  

If $f : \R^d \to \R$ is a function that has a unique minimizer on a set $C$, then we denote $\argmin_{x \in C} f(x)$ this unique minimizer.

\subsection{Sketch of proof of Thm.~\ref{thm:monotone-u2}}
\label{sec:sketch}

In this section, we give the main intuitions of the proof. The detailed proof provided in the next sections does not follow exactly this sketch for technical reasons, but we still hope that this section will help the reader to interpret the detailed proof.

On a conceptual level, Thm.~\ref{thm:monotone-u2} draws a connection between the positive lasso regularization path and the DLN dynamics. The relationship between both objects is far from obvious. We propose to understand the connection through two intermediate objects: parametric linear complementarity problems (LCPs) and LCPs with derivatives. The structure of connections is summarized below. 

\bigskip
\begin{tikzpicture}[
    every node/.style={align=center},
    bigarrow/.style={black, very thick, <->},
    vertarrow/.style={black, very thick, <->},
]

\node (lasso) at (-3,1) {
    \textbf{\large Positive lasso}\\[6pt]
    $\displaystyle
    \min_{x \geq 0}\ \text{Lasso}(x,\mu)
    $
};

\node (dln) at (3,1) {
    \textbf{\large DLN dynamics}\\[6pt]
    $\displaystyle
    \frac{\diff \var}{\diff t} = - 4 \var \circ \nabla \ell(\var) - 4 \lambda \var
    $
};

\node (lcp) at (-3,-2) {
    \textbf{\large Parametric LCP}\\[8pt]
    $\displaystyle
    \begin{cases}
    w = Mz - \mu r + (1+\mu\lambda)\mathbbb{1}, \\
    w \geq 0,\ z \geq 0,\ \langle w,z\rangle = 0.
    \end{cases}
    $
};

\node (dlcp) at (3,-2) {
    \textbf{\large LCP with derivatives}\\[8pt]
    $\displaystyle
    \begin{cases}
    w = Mz - s r + (1+s \lambda)\mathbbb{1}, \\
    w \geq 0,\ \dfrac{\diff z}{\diff s} \geq 0,\ 
    \left\langle w,\dfrac{\diff z}{\diff s}\right\rangle = 0.
    \end{cases}
    $
};

\draw[bigarrow] (lcp.east) -- (dlcp.west);

\draw[vertarrow] (lasso.south) -- (lcp.north);
\draw[vertarrow] (dln.south) -- (dlcp.north);

\end{tikzpicture}
\bigskip

In this diagram, vertical arrows are connections provided by duality, in a sense that we make precise below. Once these dual perspectives are established, the positive lasso regularization path and the DLN dynamics are transformed into a parametric LCP and a LCP with derivatives, respectively. Their connection is then much easier to grasp, as suggested by the similarity in the names and in the mathematical formulations. Note that a matching between the parametric LCP and the LCP with derivatives requires to identify $s = \mu$, i.e., the rescaled time $s$ corresponds to an inverse regularization parameter $\mu$. 

We now provide more details on each of these connections. 

\paragraph*{The lasso regularization path and the parametric LCP.} The LCP corresponds to the primal-dual formulation of the positive lasso, where $z$ is the primal variable and $w$ the dual variable. As $\Lasso(.,\mu)$ is a convex quadratic function on the constraint set $\{x \geq 0\}$, the stationarity condition writes as a linear equation, and is combined with primal feasibility $z \geq 0$, dual feasibility $w \geq 0$ and complementary slackness $\langle w, z \rangle = 0$. The lasso regularization path corresponds to the parametrization of this LCP as a function of $\mu$, hence the name \emph{parametric LCP}. This connection is detailed in Secs.~\ref{subsubsec:LCP}--\ref{sec:param-LCP}.

\paragraph*{The DLN dynamics and the LCP with derivatives.} The DLN dynamics can also be given a primal-dual interpretation, through their mirror flow interpretation (see Sec.~\ref{sec:mirror-flow}): the DLN dynamics~\eqref{eq:flow-u2-theta} can be written as a mirror flow
\begin{align*}
&\frac{\diff \nabla h(\var)}{\diff t} = - \nabla \widetilde{L}(\var) \, , &&\text{where } \widetilde{L}(\var) = \ell(\var) + \lambda \langle \mathbbb{1}, \var \rangle \, .
\end{align*}
where $h(x) = \frac{1}{4}\sum_{i=1}^{d} (x_i \log x_i - x_i)$ is the entropy potential function; $\nabla h(\var) = \frac{1}{4} \log \var$ is the dual variable. For scaling reasons as $\varepsilon \to 0$, we will actually consider a rescaled dual variable 
\begin{align*}
	w^\varepsilon(s) &= - \frac{4}{\log \frac{1}{\varepsilon}} \nabla h(x^\varepsilon(s)) = - \frac{1}{\log \frac{1}{\varepsilon}} \log x^\varepsilon(s) \, . 
\end{align*}
This rescaled version particularly fits the analysis in the rescaled time $s$. For instance, 
\begin{equation*}
	\frac{\diff w^\varepsilon}{\diff s} = -\frac{4}{\log \frac{1}{\varepsilon}} \frac{\diff t}{\diff s} \frac{\diff \nabla h(x^\varepsilon)}{\diff t} {=} \nabla \widetilde{L} (x^\varepsilon) =  M {x}^\varepsilon -r + \lambda \mathbbb{1} \, .
\end{equation*}
One can also write $x_i^\varepsilon(s) = \varepsilon^{w_i^\varepsilon(s)}$. In other words, the dual variable $w^\varepsilon(s)$ encodes the scale of $x^\varepsilon(s)$ in $\varepsilon$. At initialization, $x^\varepsilon$ is of order $\varepsilon$, thus $w^\varepsilon(0) \approx \mathbbb{1}$. (We use approximate symbols to denote relations that are rigorous in a certain sense when $\varepsilon \to 0$.) As the trajectories of $x^\varepsilon(s)$ are uniformly bounded (see Sec.~\ref{sec:uniform-bound}), we must have $w^\varepsilon(s) \gtrsim 0$. Finally, when $w_i^\varepsilon(s)$ is positive, $x_i^\varepsilon(s)$ converges to $0$ as $\varepsilon \to 0$, thus we have the (approximate) complementary slackness $\langle w^\varepsilon(s), x^\varepsilon(s) \rangle \approx 0$.

We thus have the system 
\begin{align*}
	&\frac{\diff w^\varepsilon}{\diff s} = M x^\varepsilon - r + \lambda \mathbbb{1}\, , \\
	& w^\varepsilon \gtrsim 0 \, , \qquad x^\varepsilon \geq 0 \, , \qquad \langle w^\varepsilon, x^\varepsilon \rangle \approx 0 \, .
\end{align*}
Writing $z^\varepsilon(s) = \int_0^s x^\varepsilon(u) \diff u$ and integrating the first equation, this gives the so-called LCP with derivatives \cite{kaneko1978parametric}
\begin{align*}
	& w^\varepsilon = M z^\varepsilon - s r + (1+\lambda s)\mathbbb{1} \, , \\
	& w^\varepsilon \gtrsim 0 \, , \qquad \frac{\diff z^\varepsilon}{\diff s} \geq 0 \, , \qquad \left\langle w^\varepsilon, \frac{\diff z^\varepsilon}{\diff s} \right\rangle \approx 0 \, .
\end{align*}

\paragraph*{The parametric LCP and the LCP with derivatives.} We now identify $s = \mu$ to help a correspondence. If the solution of the parametric LCP is assumed to be monotone in its parameter, then it is also a solution of the LCP with derivatives \cite{kaneko1978parametric}. The monotonicity assumption of Thm.~\ref{thm:monotone-u2} corresponds to this requirement. 

Finally, if $M$ is positive definite, the solution of the LCP with derivatives is unique (up to initialization). Indeed, consider two solutions $(w_1, z_1)$ and $(w_2, z_2)$. Then 
\begin{align*}
	\frac{\diff}{\diff s} \left(\frac{1}{2} \Vert z_1(s) - z_2(s) \Vert_M^2 \right) &= \left\langle M(z_1(s) - z_2(s)), \frac{\diff z_1}{\diff s}(s) - \frac{\diff z_2}{\diff s}(s) \right\rangle \\
	&= \left\langle w_1(s) - w_2(s), \frac{\diff z_1}{\diff s}(s) - \frac{\diff z_2}{\diff s}(s) \right\rangle \\
	&= \left\langle w_1(s), \frac{\diff z_1}{\diff s}(s) \right\rangle + \left\langle w_2(s), \frac{\diff z_2}{\diff s}(s) \right\rangle \\
	&\quad - \left\langle w_1(s), \frac{\diff z_2}{\diff s}(s) \right\rangle - \left\langle w_2(s), \frac{\diff z_1}{\diff s}(s) \right\rangle \, .
\end{align*}
By the properties of the LCP with derivatives, we have
\begin{equation*}
\left\langle w_1(s), \frac{\diff z_1}{\diff s}(s) \right\rangle = \left\langle w_2(s), \frac{\diff z_2}{\diff s}(s) \right\rangle = 0 \, .
\end{equation*}
Moreover, as $w_1(s), w_2(s) \geq 0$ and $\frac{\diff z_1}{\diff s}(s), \frac{\diff z_2}{\diff s}(s) \geq 0$, we have
\begin{equation*}
- \left\langle w_1(s), \frac{\diff z_2}{\diff s}(s) \right\rangle - \left\langle w_2(s), \frac{\diff z_1}{\diff s}(s) \right\rangle \leq 0 \, .
\end{equation*}
Thus $\Vert z_1(s) - z_2(s) \Vert_M^2$ is decreasing. This proves that the solution of the LCP with derivatives is unique up to initialization. This uniqueness enables to identify the solution of the parametric LCP with the solution of the LCP with derivatives. 

\paragraph*{Conclusion.} We have sketched the connections between the main objects of the proof. This sketch should help to interpret the results of Sec.~\ref{sec:ucarre} (and consequently of Sec.~\ref{sec:uv}). The fact that the rescaled time $s$ plays the role of an inverse regularization parameter $\mu$ is particularly clear from a comparison between the parametric LCP and the LCP with derivatives. Further, the monotonicity assumption on the lasso regularization path is required to make a connection between the parametric LCP and the LCP with derivatives. Finally, we obtain a result on the averaged trajectory (and not trajectory itself) because the LCP with derivatives is obtained by integrating the DLN dynamics.

However, this proof sketch hides several technical challenges. First, $M$ is not assumed to be invertible, which breaks the uniqueness of the solution of the LCP and of the LCP with derivatives. Second, the proof sketch uses informal approximate symbols to deal with asymptotics when $\varepsilon \to 0$. Making these approximations rigorous requires significant additional work. Finally, we deal directly with the nonmonotone case of Thm.~\ref{thm:nonmonotone-u2} (which requires to control the deviation from the proof sketch above), and then deduce Thm.~\ref{thm:monotone-u2} as a corollary.

We provide the detailed and rigorous proofs below.

\subsection{Mirror flow interpretation}
\label{sec:mirror-flow}

We recall the interpretation of the DLN dynamics as a mirror flow~\cite{li2022implicit}. Eq.~\eqref{eq:flow-u2-theta} can be interpreted as a mirror flow
\begin{align}
	\label{eq:mirror-flow-2}
	&\frac{\diff \nabla h(x^\varepsilon)}{\diff t} = - \nabla \widetilde{L}(x^\varepsilon) \, , 
\end{align}
where $\widetilde{L}(x) = \ell(x) + \lambda \langle \mathbbb{1}, x \rangle$ and the mirror map 
\begin{align*}
 h(x) = \frac{1}{4}\sum_{i=1}^{d} \left(x_i \log x_i - x_i\right)
\end{align*}
is the entropy of the vector $x > 0$. 

A central tool in analyzing mirror flows is the Bregman divergence associated to the mirror map $h$. For $x, y > 0$, it is defined as 
\begin{align}
	\label{eq:bregman-divergence}
	\begin{split}
	D(x,y) &= h(x) - h(y) - \langle \nabla h(y), x - y \rangle \\
	&= \frac{1}{4}\sum_{i=1}^{d} \left( x_i \log \frac{x_i}{y_i} - x_i + y_i \right) \, .
	\end{split}
\end{align}
Note that, by convexity of the mirror map $h$, the Bregman divergence $D$ is always nonnegative. In our case, this Bregman divergence can be interpreted as a relative entropy.

The mirror variable $\nabla h(x^\varepsilon)$ has an initialization $\nabla h(x^\varepsilon(0))$ that is diverging as $\varepsilon \to 0$: 
\begin{equation*}
	\nabla h(x^\varepsilon(0)) = \frac{1}{4} \log x^\varepsilon(0) = -\frac{1}{4} \left(\log \frac{1}{\varepsilon} \right)\mathbbb{1} + O(1) \, .
\end{equation*}
For this reason, we define a rescaled mirror variable 
\begin{align*}
	w^\varepsilon(s) &= - \frac{4}{\log \frac{1}{\varepsilon}} \nabla h(x^\varepsilon(s)) = - \frac{1}{\log \frac{1}{\varepsilon}} \log x^\varepsilon(s) \, . 
\end{align*}
We have $w^\varepsilon(0) \xrightarrow[\varepsilon \to 0]{} \mathbbb{1}$. This rescaled version particularly fits the analysis in the rescaled time $s$. For instance, 
\begin{equation}
	\label{eq:derivative-w^eps}
	\frac{\diff w^\varepsilon}{\diff s} = -\frac{4}{\log \frac{1}{\varepsilon}} \frac{\diff t}{\diff s} \frac{\diff \nabla h(x^\varepsilon)}{\diff t} \underset{\text{(Eq.~\eqref{eq:mirror-flow-2})}}{=} \nabla \widetilde{L} (x^\varepsilon) = M {x}^\varepsilon -r +\lambda \mathbbb{1} \, .
\end{equation}

\subsection{Uniform bound on the trajectories}
\label{sec:uniform-bound}

The goal of this section is to prove that the trajectories $\var^{\varepsilon}(t)$ of the DLN dynamics are bounded uniformly in $t \geq 0$ and sufficiently small $\varepsilon > 0$. 

\begin{proposition}
	\label{prop:uniform-bound}
	There exists $C = C(d, M, r, \lambda, \alpha) > 0$ and $\varepsilon_0= \varepsilon_0(\alpha) > 0$ such that 
	\begin{equation*}
		\forall \varepsilon \in (0, \varepsilon_0), \forall t \geq 0, \quad \Vert \var^\varepsilon(t) \Vert \leq C \, .
	\end{equation*}
\end{proposition}

To prove this result, we first note that as the DLN dynamics are a gradient flow, the function values must be nonincreasing along trajectories. 

\begin{lem}
	\label{lem:bound-value-L-tilde}
	For all $\varepsilon > 0$, $t \mapsto \widetilde{L}(\var^{\varepsilon}(t))$ is nonincreasing. As a consequence, there exists $C = C(M, r,\lambda, \alpha)> 0$ such that for all $0< \varepsilon \leq 1$, for all $t \geq 0$, $\widetilde{L}(\var^\varepsilon(t)) \leq C$. 
\end{lem}

\begin{proof}
	Using Eq.~\eqref{eq:flow-u2-theta}, 
	\begin{align*}
		\frac{\diff}{\diff t} \left(\widetilde{L}(\var^{\varepsilon})\right) = \sum_{i=1}^{d} \partial_i \widetilde{L}(\var^{\varepsilon}) \frac{\diff \var_i^\varepsilon}{\diff t} = -4 \sum_{i=1}^{d} \var^\varepsilon_i \left(\partial_i \widetilde{L}(\var^{\varepsilon})\right)^2 \leq 0 \, .
	\end{align*}
	This proves that $t \mapsto \widetilde{L}(\var^{\varepsilon}(t))$ is nonincreasing. 

	Moreover, the continuous function $\widetilde{L}$ is bounded by a constant $C = C(M,r,\lambda, \alpha)$ on the compact interval $\left\{\kappa \alpha^2, \, \kappa \in [0,1]\right\} \subset \R^d$. As a consequence, for all $0 < \varepsilon \leq 1$, for all $t \geq 0$,
	\begin{align*}
		\widetilde{L}(\var^{\varepsilon}(t)) \leq \widetilde{L}(\var^{\varepsilon}(0)) = \widetilde{L}(\varepsilon \alpha^2) \leq C \, . 
	\end{align*} 
\end{proof}
If $\lambda > 0$ or $M$ is positive definite, the function $\widetilde{L}$ is coercive on the set of nonnegative vectors, thus Lemma~\ref{lem:bound-value-L-tilde} implies Prop.~\ref{prop:uniform-bound}. However, the proof when $\lambda = 0$ and $M$ is not positive definite requires more work. We now assume that we are in this case. Lemma~\ref{lem:bound-value-L-tilde} only implies a bound on $Mx^\varepsilon(t)$. 
\begin{lem}
	\label{lem:bound-image-M}
	There exists $C = C(M,r,\alpha) > 0$ such that for all $0 < \varepsilon \leq 1$, for all $t \geq 0$, $\Vert M\var^\varepsilon(t) \Vert \leq C$.
\end{lem}
\begin{proof}
	Let $\var_* = \var_*(r,M)$ be the minimum norm minimizer of $\ell$. Then for any $\var \in \R^d$, 
	\begin{align*}
		\Vert M\var \Vert^2 &\leq \Vert M \Vert \Vert M^{1/2}\var \Vert^2 \leq 2\Vert M \Vert \left( \Vert M^{1/2}\var_* \Vert^2 + \Vert M^{1/2}(\var - \var_*) \Vert^2 \right) \\
		&= 2\Vert M \Vert \left( \Vert M^{1/2}\var_* \Vert^2 + 2\ell(\var) - 2\ell(\var_*) \right) 
		\\
		&= 2\Vert M \Vert \left( \Vert M^{1/2}\var_* \Vert^2 + 2\widetilde{L}(\var) - 2\ell(\var_*) \right) \, . 
	\end{align*}
	By Lemma~\ref{lem:bound-value-L-tilde}, there exists $C_1 = C_1(M,r,\alpha) > 0$ such that for all $0 < \varepsilon \leq 1$, for all $t \geq 0$, $\widetilde{L}(\var^\varepsilon(t)) \leq C_1$.
	Thus the lemma holds with
	\begin{equation*}
		C(M,r,\alpha) = 2\Vert M \Vert \left( \Vert M^{1/2}\var_* \Vert^2 + 2C_1(M,r,\alpha) - 2\ell(\var_*) \right) \, .
	\end{equation*}
\end{proof}

Our strategy consists in characterizing the flow $\var^\varepsilon(t)$ from its image $M \var^\varepsilon(t)$, and, as a consequence, showing that $\var^\varepsilon(t)$ is bounded by a bound on its image $M \var^\varepsilon(t)$. These arguments are the two lemmas below.

\begin{lem}
	\label{lem:characterization-theta-from-image}
	For all $t \geq 0$, 
	\begin{equation}
		\label{eq:characterization-theta-from-image}
		\var^\varepsilon(t) = \argmin_{\var \geq 0, \, M\var = M\var^\varepsilon(t)} D(\var, \var^\varepsilon(0)) \, ,
	\end{equation}
	where $D(\var, y)$ is defined in Eq.~\eqref{eq:bregman-divergence}. Note that we extend $D(\var, y)$ to coordinates $\var_i = 0$ by setting by continuity $0\log 0 = 0$. 
\end{lem}

\begin{proof}
A point $\var>0$ is a minimizer of $\mindot_{\var \geq 0, \, M\var = M\var^\varepsilon(t)} D(\var, \var^\varepsilon(0))$ if and only if $M \var = M \var^\varepsilon(t)$ and $\nabla_\var D(\var, \var^\varepsilon(0)) \in \Span M$. For $\var = \var^\varepsilon(t) > 0$, only the second condition needs to be checked. As $\nabla_\var D(\var, y) = \nabla h (x) - \nabla h(y)$ and by Eq.~\eqref{eq:mirror-flow-2}, we have 
\begin{align*}
	\nabla_\var D(\var^\varepsilon(t), \var^\varepsilon(0)) &= \nabla h(\var^\varepsilon(t)) - \nabla h(\var^\varepsilon(0)) \\
	&= \int_{0}^{t} \left(r - M \var^{\varepsilon}(t') \right)\diff t' \in \Span M \, . 
\end{align*}
Moreover, $\var^\varepsilon(t)$ is the unique minimizer as $D(.,.)$ is marginally strictly convex in its first variable. This justifies using the notation $\argmin$. This concludes the proof.
\end{proof}

\begin{lem}
	\label{lem:bound-argmin-divergence}
	There exists $\varepsilon_0 = \varepsilon_0(\alpha) > 0$ and $C = C(d, M) > 0$ such that for all $0 < \varepsilon \leq \varepsilon_0$, for all $y = Mu$ for some $u \in \R^d$, $u \geq  0$, denoting 
	\begin{align}
		\label{eq:argmin-bound-argmin-divergence}
		x(y) = \argmin_{x \geq 0, \,  Mx = y} D(x, \var^\varepsilon(0)) \, , 
	\end{align}
	we have 
	\begin{equation*}
		\Vert x(y) \Vert \leq C (1 + \Vert y \Vert^2) \, . 
	\end{equation*}
\end{lem}

Before we prove this last lemma, let us combine the above lemmas to prove Prop.~\ref{prop:uniform-bound}.

\begin{proof}[Proof of Proposition \ref{prop:uniform-bound}]
By Lemma~\ref{lem:characterization-theta-from-image}, 
	\begin{equation}
	\var^\varepsilon(t) = \argmin_{\var > 0, \, M\var = M\var^\varepsilon(t)} D(\var, \var^\varepsilon(0)) \, .
\end{equation}
Using Lemma~\ref{lem:bound-argmin-divergence} and then Lemma~\ref{lem:bound-image-M}, there exists $\varepsilon_0 = \varepsilon_0(\alpha) > 0$ such that for all $0 < \varepsilon < \varepsilon_0$,
\begin{equation*}
	\Vert \var^\varepsilon(t) \Vert \leq C(d,M) (1+\Vert M \var^\varepsilon(t) \Vert^2) \leq C(d,M, r, \alpha)  \, .
\end{equation*}
\end{proof}

We are now only left with the proof of Lemma~\ref{lem:bound-argmin-divergence} to finish this section. This result requires a Caratheodory theorem on cones. Let $a_i$, $i \in I$, denote a finite family of vectors in $\R^d$. The cone generated by $(a_i)_{i \in I}$ is defined as
	\begin{equation*}
		\Cone(a_i, i\in I) = \left\{ \sum_{i \in I} \lambda_i a_i \, \middle\vert \, \lambda_i \geq 0 , i \in I\right\} \, .
	\end{equation*}

\begin{theorem}[{\cite{caratheodory1911variabilitatsbereich}}]
	Let $a_1, \dots, a_k \in \R^d$ and $y \in \Cone(a_1, \dots, a_k)$. 
	
	There exists a linearly independent subfamily $a_i$, $i \in I$, $I \subset \{1, \dots, k\}$ such that $y \in \Cone(a_i, i\in I)$. 

	In particular, $y$ is in the cone generated by a subfamily of at most $d$ vectors.
\end{theorem}

To be precise, Carathéodory's theorem usually refers to the last statement. However, below, we are interested in the linear independence statement, which follows easily from the same proof. 

\begin{proof}
	As $y \in \Cone(a_1, \dots, a_k)$, $y$ can be decomposed as $y = \sum_{i\in I} \lambda_i a_i$ with $I \subset \{1, \dots, k\}$ and $\lambda_i \geq 0$, $i \in I$. Consider such a decomposition with $I$ of minimal cardinality. We will show that in this case, the $a_i$, $i \in I$ are linearly independent.

	By contradiction, assume that there exists $\mu_i$, $i \in I$, not all equal to $0$, such that $\sum_{i\in I} \mu_i a_i = 0$. Then for all $t$, we have that $y = \sum_{i \in I} (\lambda_i - t \mu_i) a_i$. We will use this degree of freedom in $t$ to remove an element from $I$, and thus contradict the minimality of~$I$.
	
	Without loss of generality, we can assume that there exists a positive $\mu_i$. (We can consider $-\mu$ instead of $\mu$.) Consider 
	\begin{equation*}
		\tau = \inf \left\{ t \geq 0 \, \middle\vert \, \lambda_i - t \mu_i = 0 \text{ for some }i \in I\right\} \, .
	\end{equation*}
	Note that $\tau < \infty$ as there exists $i \in I$ such that $\mu_i > 0$. Then for all $i \in I$, $\lambda_i - \tau \mu_i \geq 0$, with equality for at least one $i \in I$, that we denote $i_0$. Then $y = \sum_{i \in I\backslash\{i_0\}} (\lambda_i - \tau \mu_i) a_i$. Thus $I$ is not minimal, which gives a contradiction.
\end{proof}

We now use Caratheodory's theorem to show a central lemma. Let us give first the motivation. Consider a feasible linear system $Ax = y$. The minimum norm solution $x = A^\dagger y$ is bounded in norm by $\Vert A^{\dagger} \Vert \Vert y \Vert$. We now consider the same question for the system $Ax = y$ with the additional constraint $x \geq 0$ (assuming there exists at least one such solution). How can the minimum norm solution be bounded? Carathéodory's theorem enables to prove that, the non-negative minimum norm solution is, again, bounded by a constant proportional to the norm of $y$.

\begin{lem}
	\label{lem:bound-argmin-norm}
	Let $A \in \R^{d \times k}$. There exists $C = C(A) > 0$ such that for all $y = Au$ for some $u \in \R^k$, $u \geq 0$, denoting \begin{align}
		\label{eq:argmin-bound-argmin-norm}
		x(y) = \argmin_{x \geq 0, Ax = y} \Vert x \Vert \, , 
	\end{align}
	we have
	\begin{equation*}
		\Vert x(y) \Vert \leq C \Vert y \Vert \, .
	\end{equation*}
\end{lem}

\begin{proof}
	First, we note that $\Vert x \Vert$ diverges as $\Vert x \Vert \to \infty$, $x \geq 0, Ax = y$. As a consequence, the minimizer in Eq.~\eqref{eq:argmin-bound-argmin-divergence} is well-defined as the unique minimizer of a strictly convex function on a compact set. 

	In this proof, we show that the lemma holds with $C := \max_{I \subset \{1, \dots, k\}} \Vert (A_I)^\dagger \Vert$. Denote $a_1, \dots, a_k$ the columns of $A$. Let $y = Au$ for some $u \geq 0$. Then $y \in \Cone(a_1, \dots, a_k)$. By Caratheodory's theorem, there exists a linearly independent subfamily $a_i$, $i \in I$, such that $y \in \Cone(a_i, i\in I)$, i.e., $y = A_I v$ for some $v \in \R^I$, $v \geq 0$. The matrix $A_I$ has full column rank, thus $v = (A_I)^\dagger y$. Consider $x \in \R^k$ such that $x_i = v_i$ if $i \in I$, and $x_i = 0$ otherwise. Then $Ax = y$ and $x \geq 0$. Thus
	\begin{equation*}
		\Vert x(y) \Vert \leq \Vert x \Vert = \Vert v \Vert = \Vert (A_I)^\dagger y \Vert \leq \Vert (A_I)^\dagger \Vert \Vert y \Vert \leq C \Vert y \Vert \, . 
	\end{equation*}
	This completes the proof.
\end{proof}

 We are now ready to prove Lemma~\ref{lem:bound-argmin-divergence}.

\begin{proof}[Proof of Lemma~\ref{lem:bound-argmin-divergence}]
	First, we note that $D(x, x^\varepsilon(0))$ diverges as $\Vert x \Vert \to \infty$, $x \geq 0$. As a consequence, the minimizer in Eq.~\eqref{eq:argmin-bound-argmin-divergence} is well-defined as the unique minimizer of a strictly convex function on a compact set. 

	The only difference between Lemma~\ref{lem:bound-argmin-divergence} and Lemma~\ref{lem:bound-argmin-norm} is that the former minimizes $D(\var, \var^\varepsilon(0))$ while the latter minimizes $\Vert \var \Vert$. However, the two quantities are comparable in the sense that there exists $0<\varepsilon_0 = \varepsilon_0(\alpha)<1$, $C_1 = C_1(d)$, $C_2 = C_2(d) > 0$ such that for all $0 < \varepsilon \leq \varepsilon_0$, for all $\var \geq 0$, 
	\begin{equation}
		\label{eq:aux-1}
		\frac{1}{8} \left(\log \frac{1}{\varepsilon}\right)\Vert \var \Vert - C_1 \leq D(\var, \var^\varepsilon(0)) \leq   C_2 \left(\log \frac{1}{\varepsilon}\right) \Vert \var \Vert + \frac{1}{4}\Vert \var \Vert^2 + 1 \, . 
	\end{equation}
	Indeed, 
	\begin{align*}
		D(\var, \var^\varepsilon(0)) &= \frac{1}{4}\sum_{i=1}^{d} \left(\var_i \log \frac{\var_i}{\varepsilon\alpha_i^2} - \var_i + \varepsilon \alpha_i^2\right) \\
		&= \frac{1}{4}\sum_{i=1}^{d} \var_i \log \var_i + \frac{1}{4}\left(\log \frac{1}{\varepsilon} - 1\right) \Vert \var \Vert_1 + \frac{1}{4} \sum_{i=1}^{d} x_i \log \frac{1}{\alpha_i^2} + \frac{\varepsilon}{4} \sum_{i=1}^{d} \alpha_i^2\, .
	\end{align*}
	For the lower-bound, we use that $u \mapsto u\log u$ is bounded from below on $\R_{\geq 0}$. Thus there exists $0<\varepsilon_1 = \varepsilon_1(\alpha)<1$ such that, for all $\varepsilon < \varepsilon_1$, 
	\begin{align*}
		D(\var, \var^\varepsilon(0)) \geq -C_1(d) + \frac{1}{8} \left(\log \frac{1}{\varepsilon}\right)\Vert \var \Vert \, . 
	\end{align*}
	For the upper-bound, we have
	\begin{align*}
		D(\var, \var^\varepsilon(0)) &= \frac{1}{4}\sum_{i=1}^{d} \left(\var_i \log \frac{\var_i}{\varepsilon\alpha_i^2} - \var_i + \varepsilon \alpha_i^2\right) \\
		&\leq \frac{1}{4}\sum_{i=1}^{d} \var_i \left(\log \var_i + \log \frac{1}{\varepsilon} + \log \frac{1}{\alpha_i^2}\right) + \frac{\varepsilon}{4} \sum_{i=1}^{d} \alpha_i^2\, .
	\end{align*}
	We use $\log y \leq y$ for $y > 0$. Thus there exists $0<\varepsilon_2 = \varepsilon_2(\alpha)<1$, $C_2 = C_2(d)$ such that, for all $\varepsilon < \varepsilon_2$,
	\begin{align*}
		D(\var, \var^\varepsilon(0)) &\leq \frac{1}{4}\Vert \var \Vert^2 +\frac{1}{2}  \left(\log \frac{1}{\varepsilon}\right) \sum_{i=1}^{d} \var_i + 1 \leq \frac{1}{4}\Vert \var \Vert^2 + C_2(d) \left(\log \frac{1}{\varepsilon}\right) \Vert \var \Vert + 1 \, .
	\end{align*}
Thus Eq.~\eqref{eq:aux-1} holds for $\varepsilon \leq \varepsilon_0(\alpha) := \min(\varepsilon_1(\alpha), \varepsilon_2(\alpha))$. 

We now define $\beta(y) = \argmin_{\beta \geq 0, M\beta = y} \Vert \beta \Vert$. By Lemma~\ref{lem:bound-argmin-norm}, $\Vert \beta(y) \Vert \leq C_3(M) \Vert y \Vert$. Then we have, 
\begin{align*}
	\frac{1}{8} \left(\log \frac{1}{\varepsilon}\right)\Vert x(y) \Vert &\leq C_1(d) + D(\var(y), \var^\varepsilon(0)) \qquad \text{by Eq.~\eqref{eq:aux-1}}\\
	&\leq C_1(d) + D(\beta(y), \var^\varepsilon(0)) \qquad \text{by the variational definition of $x(y)$}\\
	&\leq C_4(d) + C_2(d) \left(\log \frac{1}{\varepsilon}\right) \Vert \beta(y) \Vert + \frac{1}{4}\Vert \beta(y) \Vert^2  \qquad \text{by Eq.~\eqref{eq:aux-1}}\\
	&\leq C_4(d) + C_5(d,M) \left(\log \frac{1}{\varepsilon}\right) \Vert y \Vert + C_6(d,M) \Vert y \Vert^2  \, . 
\end{align*}
Assuming that $\varepsilon_0$ is small enough so that $\frac{1}{8}\log \frac{1}{\varepsilon_0} \geq 1$, we have 
\begin{align*}
	\Vert x(y) \Vert &\leq  C_4(d) + 8C_5(d,M)  \Vert y \Vert + C_6(d,M) \Vert y \Vert^2 \leq C(d,M) (1 + \Vert y \Vert^2) \, . 
\end{align*}

\end{proof}

\subsection{The positive lasso and linear complementarity problems}
\label{subsubsec:LCP}

The positive lasso~\eqref{eq:positive-lasso} is in fact a constrained quadratic optimization problem: $\mindot_{x \geq 0} \frac{1}{2} \langle x, M x \rangle - \langle r, x \rangle + \left(\lambda+\frac{1}{\mu}\right) \langle \mathbbb{1}, x \rangle$. As a consequence, its primal-dual formulation is a linear complementarity problem \cite{cottle2009linear}.

\begin{proposition}
	\label{prop:lcp-optim}
    Let $x \in \R^d$. The vector $x$ is a minimizer of the optimization problem~\eqref{eq:positive-lasso} if and only if there exists $v \in \R^d$ such that $(v,x)$ is a solution of
	\begin{align}
		\label{eq:first_LCP}
		\begin{split}
			v &= -r + \left(\lambda+\frac{1}{\mu}\right) \mathbbb{1} + M x \, , \\
			v & \geq 0 \, , \ x \geq 0\, , \  \langle v, x \rangle = 0 ~.
		\end{split}
		\end{align}
		The system above has the form of a \emph{linear complementarity problem} (LCP) \cite{cottle2009linear}. 
\end{proposition}

This result is classical and its proof follows from convex duality, see \cite{boyd2004convex}, Sec.~5. We recall it for the sake of completeness.

\begin{proof}
	For simplicity, we denote $q = -r + \left(\lambda+\frac{1}{\mu}\right) \mathbbb{1}$. The Lagrangian associated to~\eqref{eq:positive-lasso} is 
	\begin{equation*}
		L(x, v) = \Lasso(x,\mu) - \langle v, x \rangle = \frac{1}{2} \langle x, M x \rangle + \langle q, x \rangle - \langle v, x \rangle
	\end{equation*}
	where $v \in \R^d$ is the Lagrange multiplier associated to the constraint $x \geq 0$. As the optimization problem \eqref{eq:positive-lasso} is convex, the KKT conditions are necessary and sufficient for optimality. The stationarity condition is 
	\begin{equation*}
		0 = \nabla_{x} L(x,v) = q + M x - v \, ,
	\end{equation*} 
	the feasibility conditions are $x \geq 0$ and $v \geq 0$, and the complementary slackness condition is $\langle v, x \rangle = 0$. This proves Prop.~\ref{prop:lcp-optim}. 
	\end{proof}

We now study the existence and uniqueness of the solution of the LCP~\eqref{eq:first_LCP}. Note that if $M$ is positive definite, there exists a unique solution $(v,x)$, see \cite{cottle2009linear}, Thm.~3.1.6. In the more general case where $M$ is positive semidefinite, existence still holds but uniqueness is more subtle. For instance, the positive lasso~\eqref{eq:positive-lasso} might not have a unique minimizer, thus we might not have uniqueness of the solution in~$x$. However, we have uniqueness in $v$. 

\begin{proposition}
\label{prop:lcp-uniqueness}
There exists a solution of the LCP \eqref{eq:first_LCP}. Moreover, the solution of the LCP is unique in $v$ in the sense that if $(v,x)$ and $(v',x')$ are two solutions of the LCP, then $v=v'$. 
\end{proposition}

\begin{proof}
\emph{Existence.}
By \cite{eaves1971linear}, it suffices to show that \begin{equation*}
\Lasso(\var, \mu) = \ell(\var) + \left(\lambda+\frac{1}{\mu}\right) \langle \mathbbb{1}, \var \rangle
\end{equation*}
is bounded from below on the set $\{\var \, \vert \, \var \geq 0\}$. The first term $\ell(\var)$ is bounded from below on $\R^d$ (as $r \in \Span(M)$), and the second term is nonnegative on $\{\var \, \vert \, \var \geq 0\}$.

\emph{Uniqueness.}
    The result is provided in \cite{cottle2009linear}, Thm.~3.1.7(d). 
\end{proof}

\subsection{The positive lasso regularization path and parametric linear complementarity problems}
\label{sec:param-LCP}

In Sec.~\ref{sec:ucarre}, we consider a solution ${x}(\mu)$ of the positive lasso with varying inverse regularization parameter $\mu > 0$. This describes the lasso regularization path. In this section, we draw a connection with so-called \emph{parametric} LCPs. 

From Prop.~\ref{prop:lcp-optim}, there exists $v(\mu)$ such that $(v(\mu),{x}(\mu))$ is a solution of the LCP 
\begin{align}
	\label{eq:second_LCP}
	\begin{split}
		v(\mu) &= -r + \left(\lambda+\frac{1}{\mu}\right) \mathbbb{1} + M {x}(\mu) \, ,  \\
		v(\mu) & \geq 0 \, , \ {x}(\mu) \geq 0\, , \  \langle v(\mu), {x}(\mu) \rangle = 0 ~.
	\end{split}
\end{align}
Denote $w(\mu) = \mu v(\mu)$ and $z(\mu) = \mu {x}(\mu)$. Then
\begin{align}
	\label{eq:parametric-LCP}
	\begin{split}
		w(\mu) &= -\mu r + \left(1+\mu\lambda\right)\mathbbb{1} + M z(\mu) \, ,  \\
		w(\mu) & \geq 0 \, , \ z(\mu) \geq 0\, , \  \langle w(\mu), z(\mu) \rangle = 0 ~.
	\end{split}
\end{align}
This is a \emph{parametric} LCP, where the parameter is $\mu$, see \cite{cottle2009linear}, Sec.~4.5. 

Note that while the LCP \eqref{eq:second_LCP} is defined only for $\mu > 0$, the LCP \eqref{eq:parametric-LCP} is also defined for $\mu=0$. This extension is rather trivial, the solution of both LCPs being simple for $\mu$ small enough, as shown in the following lemma. 

\begin{lem}
	\label{lem:s-small}
	Assume $0 \leq \mu < \frac{1}{\max(\Vert r - \lambda \mathbbb{1} \Vert_\infty, 1)}$. Then the LCP \eqref{eq:parametric-LCP} has a unique solution $(w(\mu), z(\mu)) = ((1+\mu\lambda)\mathbbb{1}-\mu r, 0)$. If further $\mu > 0$, this means that the LCP \eqref{eq:second_LCP} has a unique solution $(v(\mu), {x}(\mu)) = \left(-r + \left(\lambda+\frac{1}{\mu}\right)\mathbbb{1}, 0\right)$.
\end{lem}

\begin{proof}
	$(w(\mu), z(\mu)) = ((1+\mu\lambda)\mathbbb{1}-\mu r, 0)$ is a solution of \eqref{eq:parametric-LCP}. As the solution is unique in $w$, $(1+\mu\lambda)\mathbbb{1}-\mu r$ is the unique solution in $w$. But as $(1+\mu\lambda)\mathbbb{1}-\mu r > 0$, this imposes that $z(\mu) =0$ is the unique solution in $z$ by complementarity. 
\end{proof}

We now prove a regularity result on the solution of the parametric LCP~\eqref{eq:parametric-LCP}.

\begin{lem}
	\label{lem:lcp-parametric-lip}
	Let $w(\mu)$ be the unique solution in $w$ of the parametric LCP~\eqref{eq:parametric-LCP}. Then $\mu \geq 0 \mapsto w(\mu)$ is absolutely continuous, $\frac{\diff }{\diff \mu}\left(\frac{1}{1+\mu\lambda}w(\mu)\right) \in \Span M$ and $\left\Vert \frac{\diff }{\diff \mu}\left(\frac{1}{1+\mu\lambda}w(\mu)\right) \right\Vert_{M^\dagger} \leq \Vert r \Vert_{M^\dagger}$. 
\end{lem}

	\begin{proof}
Let $\mu, \mu' \geq 0$ and $(w,z) = (w(\mu), z(\mu))$, $(w',z') = (w(\mu'), z(\mu'))$ denote solutions of the parametric LCP \eqref{eq:parametric-LCP} for the respective values $\mu, \mu'$. Then, we have 
		\begin{align*}
		\frac{1}{1+\mu\lambda}w &=  - \frac{\mu}{1+\mu\lambda} r + \mathbbb{1}+ M\frac{1}{1+\mu\lambda}z\, , \\ 
		\frac{1}{1+\mu'\lambda}w' &=  - \frac{\mu'}{1+\mu'\lambda} r + \mathbbb{1}+ M\frac{1}{1+\mu'\lambda}z' \, .
	\end{align*}
		Hence,
		\begin{align*}
			\frac{1}{1+\mu\lambda}w - \frac{1}{1+\mu'\lambda}w' &= -\left(\frac{\mu}{1+\mu\lambda}-\frac{\mu'}{1+\mu'\lambda}\right)r + M\left(\frac{1}{1+\mu\lambda}z - \frac{1}{1+\mu'\lambda}z'\right) \\
			&\in \Span M \, , 
		\end{align*}
		and 
		\begin{align*}
			&\left\|\frac{1}{1+\mu\lambda}w - \frac{1}{1+\mu'\lambda}w'\right\|^2_{M^{\dagger}} \\
			&\qquad= \Big\langle \frac{1}{1+\mu\lambda}w - \frac{1}{1+\mu'\lambda}w'  , -\left(\frac{\mu}{1+\mu\lambda}-\frac{\mu'}{1+\mu'\lambda}\right)r \\ 
			&\hspace{5cm}+ M\left(\frac{1}{1+\mu\lambda}z - \frac{1}{1+\mu'\lambda}z'\right)\Big\rangle_{M^\dagger} \\ 
			&\qquad=  -\left(\frac{\mu}{1+\mu\lambda}-\frac{\mu'}{1+\mu'\lambda}\right)\left\langle \frac{1}{1+\mu\lambda}w - \frac{1}{1+\mu'\lambda}w'  , r\right\rangle_{M^\dagger} \\ 
			&\qquad\qquad+ \left\langle \frac{1}{1+\mu\lambda}w - \frac{1}{1+\mu'\lambda}w' , \frac{1}{1+\mu\lambda}z - \frac{1}{1+\mu'\lambda}z' \right\rangle  \\
		&\qquad=  -\left(\frac{\mu}{1+\mu\lambda}-\frac{\mu'}{1+\mu'\lambda}\right)\left\langle \frac{1}{1+\mu\lambda}w - \frac{1}{1+\mu'\lambda}w'  , r\right\rangle_{M^\dagger} \\ 
			&\qquad\qquad+ \frac{1}{(1+\mu\lambda)^2}\left\langle w  , z \right\rangle - \frac{1}{(1+\mu\lambda)(1+\mu'\lambda)} \left\langle w  ,  z' \right\rangle \\ 
			&\qquad\qquad- \frac{1}{(1+\mu\lambda)(1+\mu'\lambda)} \left\langle w' , z \right\rangle + \frac{1}{(1+\mu'\lambda)^2}\left\langle w' , z' \right\rangle \, . 
		\end{align*}
		We can use H\"older's inequality to bound the first term. Furthermore, by complementarity slackness $\langle z, w \rangle = \langle z', w' \rangle = 0$. Finally, all vectors $w, w',z, z'$ are nonnegative, thus $\langle w',z \rangle, \langle w,z' \rangle \geq 0$. Thus, we have
		\begin{align*}
			&\left\|\frac{1}{1+\mu\lambda}w - \frac{1}{1+\mu'\lambda}w'\right\|^2_{M^{\dagger}} \\
			&\qquad\leq \left\vert \frac{\mu}{1+\mu\lambda}-\frac{\mu'}{1+\mu'\lambda}\right\vert \left\Vert \frac{1}{1+\mu\lambda}w - \frac{1}{1+\mu'\lambda}w' \right\Vert_{M^{\dagger}} \|r\|_{M^{\dagger}} \, .
		\end{align*}
		As $\mu \mapsto \frac{\mu}{1+\mu\lambda}$ is $1$-Lipschitz, we have
		\begin{align*}
			\left\|\frac{1}{1+\mu\lambda}w - \frac{1}{1+\mu'\lambda}w'\right\|^2_{M^{\dagger}} \leq \vert \mu - \mu' \vert \|r\|_{M^{\dagger}} \, .
		\end{align*}
		This proves that the map $\mu \mapsto \frac{1}{1+\mu\lambda}w(\mu)$ is Lipschitz continuous, thus absolutely continuous. Thus $\mu \mapsto w(\mu)$ is also absolutely continuous. As for all $\mu, \mu'$, $\frac{1}{1+\mu\lambda}w - \frac{1}{1+\mu'\lambda}w' \in \Span M$, we have $\frac{\diff }{\diff \mu} \left(\frac{1}{1+\mu\lambda} w(\mu)\right) \in \Span M$. The inequality above also shows that $\left\Vert \frac{\diff }{\diff \mu}\left(\frac{1}{1+\mu\lambda}w(\mu)\right) \right\Vert_{M^\dagger} \leq \Vert r \Vert_{M^\dagger}$. 
	\end{proof}

\subsection{Proof of Theorem \ref{thm:nonmonotone-u2}}
\label{sec:proof-nonmonotone-u2}

For all $\mu > 0$, ${x}(\mu)$ is a minimizer of $\mindot_{{x}\geq 0} \Lasso\left({x},\mu\right)$. We use the notations of Sec.~\ref{sec:param-LCP}: by Prop.~\ref{prop:lcp-optim}, there exists $v(\mu)$ such that $(v(\mu),{x}(\mu))$ is a solution of the LCP \eqref{eq:second_LCP}. Denote $w(\mu) = \mu v(\mu)$ and $z(\mu) = \mu {x}(\mu)$. Then $(w(\mu), z(\mu))$ is a solution of the parametric LCP \eqref{eq:parametric-LCP}. We extend this definition of $(w,z)$ with $(w(0), z(0)) = (\mathbbb{1}, 0)$, which is a solution of the LCP \eqref{eq:parametric-LCP} for $\mu = 0$.

In this proof, the solution $(w(s), z(s))$ of the parametric LCP taken at $\mu = s$ will be compared to $(w^\varepsilon(s), z^\varepsilon(s))$ where $w^\varepsilon(s)$ is defined in Sec.~\ref{sec:mirror-flow} and 
\begin{align*}
	z^\varepsilon(s) &= s \overline{x}^\varepsilon(s) = \int_{0}^{s}   x^\varepsilon(u)\diff u \, .
\end{align*}
Note that
\begin{align}
	\frac{\diff z^\varepsilon}{\diff s} &= x^\varepsilon(s) \, . \label{eq:derivative-z-epsilon}
\end{align}
Integrating Eq.~\eqref{eq:derivative-w^eps}, we obtain
\begin{equation}
	\label{eq:linear-lcp-eps}
	w^\varepsilon(s) = w^\varepsilon(0) -sr + M z^\varepsilon(s) + s \lambda \mathbbb{1} \, . 
\end{equation}
This equation has a similarity with the first equation of the LCP~\eqref{eq:parametric-LCP}. 

We now seek to bound 
\begin{equation*}
	\Delta^\varepsilon(s) := \frac{1}{2}\Vert z^\varepsilon(s) - z(s) \Vert^2_M \, .  
\end{equation*}
At $s=0$, we have by definition, $z^\varepsilon(0) = 0$ and by Lemma \ref{lem:s-small}, $z(0) = 0$. Thus $\Delta^\varepsilon(s) = 0$. 

We compute the derivative of $\Delta^\varepsilon(s)$:
\begin{align*}
	\frac{\diff \Delta^\varepsilon}{\diff s} &= \left\langle  \frac{\diff z^\varepsilon(s)}{\diff s} - \frac{\diff z(s)}{\diff s},M(z^\varepsilon(s) - z(s))  \right\rangle 
\end{align*}
We use Eqs.~\eqref{eq:derivative-z-epsilon}, \eqref{eq:linear-lcp-eps} and the first equation of \eqref{eq:parametric-LCP}:
\begin{align}
	\label{eq:derivative-Delta-epsilon}
	\begin{split}
	\frac{\diff \Delta^\varepsilon}{\diff s} &= \left\langle   x^\varepsilon(s) - \frac{\diff z(s)}{\diff s}, w^\varepsilon(s) - w(s) + \mathbbb{1} - w^\varepsilon(0) \right\rangle \\
	&= \left\langle   x^\varepsilon(s) , w^\varepsilon(s) \right\rangle - \left\langle   x^\varepsilon(s) ,  w(s) \right\rangle - \left\langle    \frac{\diff z(s)}{\diff s}, w^\varepsilon(s)  \right\rangle \\&\qquad+ \left\langle    \frac{\diff z(s)}{\diff s},  w(s)  \right\rangle 
	+ \left\langle   x^\varepsilon(s) ,  \mathbbb{1} - w^\varepsilon(0) \right\rangle - \left\langle    \frac{\diff z(s)}{\diff s}, \mathbbb{1} - w^\varepsilon(0) \right\rangle\, .
	\end{split}
\end{align}
We now upper-bound each term separately:
\begin{itemize}
	\item The function $u > 0 \mapsto u \log u$ is uniformly lower bounded by a universal constant thus 
	\begin{align*}
		\left\langle   x^\varepsilon(s) , w^\varepsilon(s) \right\rangle &= - \frac{1}{\log \frac{1}{\varepsilon}}\sum_{i=1}^{d} x_i^\varepsilon(s) \log x_i^\varepsilon(s) \leq \frac{C(d)}{\log \frac{1}{\varepsilon}} \, .
	\end{align*}
	\item We have $x^\varepsilon(s) \geq 0$ and $w(s) \geq 0$ thus $-\langle x^\varepsilon(s), w(s) \rangle \leq 0$. 
	\item If $\frac{\diff z_i(s)}{\diff s} \geq 0$, we use that by Prop.~\ref{prop:uniform-bound}, 
	\begin{equation*}
		w_i^\varepsilon(s) = -\frac{1}{\log \frac{1}{\varepsilon}} \log x_i^\varepsilon(s) \geq - \frac{C(d,M,r,\lambda,\alpha)}{\log \frac{1}{\varepsilon}} \, .
	\end{equation*}
	If $\frac{\diff z_i(s)}{\diff s} \leq 0$, we use Eq.~\eqref{eq:derivative-w^eps} and Prop.~\ref{prop:uniform-bound} to obtain 
	\begin{equation*}
		w_i^\varepsilon(s) \leq C(d,M,r,\lambda,\alpha) (1+s) \, . 
	\end{equation*}
	Thus 
	\begin{align*}
		&- \left\langle    \frac{\diff z(s)}{\diff s}, w^\varepsilon(s)  \right\rangle \\
		&\qquad\leq C(d,M,r,\lambda,\alpha) \sum_{i=1}^{d} \left[\frac{1}{\log \frac{1}{\varepsilon}} \left(\frac{\diff z_i(s)}{\diff s}\right)_+ +  (1+s) \left(\frac{\diff z_i(s)}{\diff s}\right)_-\right] \, . 
	\end{align*}
	Recall that 
	\begin{equation*}
		z^\downarrow(s) = \sum_{i=1}^{d}\int_{0}^{s} (1+u)\left(\frac{\diff z_i(u)}{\diff u}\right)_- \diff u \, .
	\end{equation*}
	Denote also \begin{equation*}
	z^\uparrow(s) = \sum_{i=1}^{d}\int_{0}^{s} \left(\frac{\diff z_i(u)}{\diff u}\right)_+ \diff u \ .
	\end{equation*} 
	With these notations, we can write
\begin{equation*}
	- \left\langle    \frac{\diff z(s)}{\diff s}, w^\varepsilon(s)  \right\rangle \leq C(d,M,r,\lambda,\alpha) \left[\frac{1}{\log \frac{1}{\varepsilon}} \frac{\diff z^\uparrow(s)}{\diff s} +   \frac{\diff z^\downarrow(s)}{\diff s}\right] \, . 
\end{equation*}

	\item Let $i \in \{1,\dots,d\}$. We show that $\frac{\diff z_i(s)}{\diff s} w_i(s) = 0$. From Eqs.~\eqref{eq:parametric-LCP}, we have $w_i(s) z_i(s) =0$. If $w_i(s) = 0$, then the conclusion is obvious. Otherwise $z_i(s) = 0$. Differentiating $w_i(s) z_i(s) =0$, we obtain $\frac{\diff w_i(s)}{\diff s} z_i(s) + w_i(s) \frac{\diff z_i(s)}{\diff s} = 0$. As $z_i(s) = 0$, this gives the desired conclusion. Thus, we have $\left\langle    \frac{\diff z(s)}{\diff s}, w(s)  \right\rangle = 0$.

	\item Note that $\mathbbb{1} - w^\varepsilon(0) = \mathbbb{1}+ \frac{1}{\log \frac{1}{\varepsilon}} \log x^\varepsilon(0) = \frac{2}{\log \frac{1}{\varepsilon}} \log \vert \alpha \vert$. This enables to bound the last two terms, using again Prop.~\ref{prop:uniform-bound}. 
\end{itemize}
We thus obtain that 
\begin{align*}
	\frac{\diff \Delta^\varepsilon(s)}{\diff s} \leq C(d,M,r,\lambda,\alpha) \left[\frac{1}{\log \frac{1}{\varepsilon}} \left(1+\frac{\diff z^\uparrow(s)}{\diff s} \right) +   \frac{\diff z^\downarrow(s)}{\diff s}\right] \, , 
\end{align*}
and thus
\begin{equation}
	\label{eq:bound-Delta}
	\Delta^\varepsilon(s) \leq
	C(d,M,r,\lambda,\alpha) \left[\frac{1}{\log \frac{1}{\varepsilon}}  \left(s + z^\uparrow(s)\right) + z^\downarrow(s)\right] \, . 
\end{equation}
At this point, we have bounded $\Vert z^\varepsilon(s) - z(s) \Vert^2_M$, or, equivalently, $\Vert \overline{x}^\varepsilon(s) - {x}(s) \Vert^2_M$, where ${x}(s)$ is a minimizer of the positive lasso with inverse regularization $\mu = s$. However, as $M$ might not be positive definite, we can not conclude directly. 

The positive lasso objective function $\Lasso(x,\mu)$ is a quadratic function on the set of nonnegative vectors $x \geq 0$: 
\begin{align*}
	\Lasso\left(x,\mu\right) = g_\mu(x) :=  \frac{1}{2} \langle x, M x \rangle - \langle r, x \rangle + \left(\lambda+\frac{1}{\mu}\right) \langle \mathbbb{1}, x \rangle \, .
\end{align*}
As a consequence, we can compute $\Lasso(\overline{x}^\varepsilon(s), s) =  g_s(\overline{x}^\varepsilon(s))$	using the second-order approximation of $g_s$ at ${x}(s)$:
\begin{align*}
	&\Lasso\left(\overline{x}^\varepsilon(s), s\right) - \PosLasso_*\left(s\right) = g_s(\overline{x}^\varepsilon(s)) - g_s({x}(s)) \\
	&\qquad= \left\langle \nabla g_s({x}(s)), \overline{x}^\varepsilon(s) - {x}(s) \right\rangle + \frac{1}{2} \Vert \overline{x}^\varepsilon(s) - {x}(s) \Vert^2_M \\
	&\qquad= \left\langle M {x}(s) - r + \left(\lambda+\frac{1}{s}\right) \mathbbb{1}, \overline{x}^\varepsilon(s) - {x}(s) \right\rangle + \frac{1}{2} \Vert \overline{x}^\varepsilon(s) - {x}(s) \Vert^2_M \\
	&\qquad= \left\langle v(s), \overline{x}^\varepsilon(s) - {x}(s) \right\rangle + \frac{1}{2} \Vert \overline{x}^\varepsilon(s) - {x}(s) \Vert^2_M \\ 
	&\qquad = \frac{1}{s^2} E^\varepsilon(s )
\end{align*}
where $E^\varepsilon(s) = \langle w(s), z^\varepsilon(s)- z(s) \rangle + \Delta^\varepsilon(s)$. We compute 
\begin{align*}
	&\frac{\diff }{\diff s}\left(\frac{1}{1+s\lambda} E^\varepsilon(s)\right) = \\
	&\qquad\left\langle \frac{\diff }{\diff s}\left(\frac{1}{1+s\lambda}w(s)\right), z^\varepsilon(s) - z(s) \right\rangle + \frac{1}{1+s\lambda}\left\langle w(s), x^\varepsilon(s) \right\rangle\\
	&\qquad- \frac{1}{1+s\lambda}\left\langle w(s) , \frac{\diff z(s)}{\diff s} \right\rangle + \frac{1}{1+s\lambda}\frac{\diff \Delta^\varepsilon(s)}{\diff s} + \frac{\diff}{\diff s} \left(\frac{1}{1+s\lambda
	}\right)\Delta^\varepsilon(s)\, .
\end{align*}
For the first term, we use the Cauchy-Schwarz inequality and Lemma~\ref{lem:lcp-parametric-lip}. Moreover, several terms cancel out with the expression of the derivative of $\Delta^\varepsilon(s)$ in Eq.~\eqref{eq:derivative-Delta-epsilon}. The last term is nonpositive. We obtain
\begin{align*}
	&\frac{\diff }{\diff s}\left(\frac{1}{1+s\lambda} E^\varepsilon(s)\right) \leq \Vert r \Vert_{M^\dagger} \Vert z^\varepsilon(s) - z(s) \Vert_{M} \\
	&\qquad+ \frac{1}{1+s\lambda}\Big[\left\langle x^\varepsilon(s), w^\varepsilon(s) \right\rangle - \left\langle \frac{\diff z(s)}{\diff s} , w^\varepsilon(s) \right\rangle \\
	&\hspace{2cm}+ \left\langle x^\varepsilon(s), \mathbbb{1} - w^\varepsilon(0)\right\rangle - \left\langle \frac{\diff z(s)}{\diff s}, \mathbbb{1} - w^\varepsilon(0) \right\rangle \Big]\, . 
\end{align*}
All of these terms have been bounded above. We obtain
\begin{align*}
	&\frac{\diff }{\diff s}\left(\frac{1}{1+s\lambda} E^\varepsilon(s)\right) \\
	&\qquad\leq C(d,M,r,\lambda,\alpha) \left[\Delta^\varepsilon(s)^{1/2} + \frac{1}{1+s\lambda}\left(\frac{1}{\log \frac{1}{\varepsilon}} \left(1+\frac{\diff z^\uparrow(s)}{\diff s} \right) +   \frac{\diff z^\downarrow(s)}{\diff s}\right)\right] \\
	&\qquad\leq C(d,M,r,\lambda,\alpha) \left[\Delta^\varepsilon(s)^{1/2} + \frac{1}{\log \frac{1}{\varepsilon}} \left(1+\frac{\diff z^\uparrow(s)}{\diff s} \right) +   \frac{\diff z^\downarrow(s)}{\diff s}\right]\, .
\end{align*}
Integrating between $0$ and $s$ and using Eq.~\eqref{eq:bound-Delta}, we obtain 
\begin{align*}
	E^\varepsilon(s) 
	&\leq C(d,M,r,\lambda,\alpha) (1+s\lambda)\Bigg[ s \left(\frac{1}{\log \frac{1}{\varepsilon}}  \left(s + z^\uparrow(s)\right) +  z^\downarrow(s)\right)^{1/2} \\
	&\qquad\qquad\qquad\qquad\qquad\qquad\qquad+   \frac{1}{\log \frac{1}{\varepsilon}}  \left(s + z^\uparrow(s)\right) +  z^\downarrow(s) \Bigg] \, . 
\end{align*}
This allows us to conclude
\begin{align*}
	&\limsup_{\varepsilon \to 0}\Lasso\left(\overline{x}^\varepsilon(s), s\right) - \PosLasso_*\left(s\right) = \limsup_{\varepsilon \to 0}\frac{1}{s^2}  E^\varepsilon(s ) \\ 
	&\qquad\leq C(d,M,r,\lambda, \alpha) (1+s\lambda)\left[\frac{z^\downarrow(s)^{1/2}}{s} + \frac{z^\downarrow(s)}{s^2} \right] \\
	&\qquad= C(d,M,r,\lambda, \alpha) \eta(\lambda,s,z^\downarrow(s))\, . 
\end{align*}

\subsection{Proof of Theorem \ref{thm:monotone-u2}}
\label{sec:proof-monotone-u2}

	Thm.~\ref{thm:monotone-u2} is an application of Thm.~\ref{thm:nonmonotone-u2}. However, to apply Thm.~\ref{thm:nonmonotone-u2}, we need to prove that $\mu \mapsto {x}(\mu)$ is absolutely continuous. This is proven in the following lemma (in combination with Lemma~\ref{lem:s-small}).

\begin{lem}
	Under the assumptions of Thm.~\ref{thm:monotone-u2}, the function $\mu \geq 0 \mapsto z(\mu)$ is absolutely continuous on compact subsets of $[0,+\infty)$. 

\end{lem}

\begin{proof}
	From Lemma~\ref{lem:lcp-parametric-lip}, $\mu\geq 0 \mapsto w(\mu)$ is locally Lipschitz continuous. Using the linear relationship $\frac{1}{1+\mu\lambda}w(\mu) = -\frac{\mu}{1+\mu\lambda} r + \mathbbb{1} + M\frac{1}{1+\mu\lambda}z(\mu)$, this implies that $\mu \geq 0 \mapsto P_{\Span M} z(\mu)$ is locally Lipschitz continuous, where $\P_{\Span M}$ denotes the orthogonal projection onto $\Span M$.
	
	Denote also $\P_{\ker M}$ denotes the orthogonal projection onto $\ker M$. We have the complementarity slackness $\langle w(\mu), z(\mu) \rangle = 0$. Recall that $\frac{1}{1+\mu\lambda}w(\mu) \in \mathbbb{1} + \Span M$. Thus we obtain 
	\begin{align*}
		\begin{split}
		0 &= \left\langle \frac{1}{1+\mu\lambda}w(\mu), z(\mu) \right\rangle \\
		&= \left\langle \frac{1}{1+\mu\lambda}w(\mu), \P_{\Span M} z(\mu) \right\rangle + \left\langle \frac{1}{1+\mu\lambda}w(\mu) ,\P_{\ker M} z(\mu) \right\rangle \\
		&= \left\langle \frac{1}{1+\mu\lambda}w(\mu), \P_{\Span M} z(\mu) \right\rangle + \left\langle \mathbbb{1} ,\P_{\ker M} z(\mu) \right\rangle \, .
		\end{split}
	\end{align*}
	As $\mu \geq 0 \mapsto \frac{1}{1+\mu\lambda}w(\mu)$ and $\mu\geq 0 \mapsto \P_{\Span M} z(\mu)$ are locally Lipschitz continuous, the dot product $\mu \geq 0 \mapsto \left\langle \frac{1}{1+\mu\lambda}w(\mu), \P_{\Span M} z(\mu) \right\rangle$ is also locally Lipschitz continuous. As a consequence, the above inequality implies that $\mu \geq 0 \mapsto \langle \mathbbb{1}, \P_{\ker M} z(\mu) \rangle$ is locally Lipschitz continuous.

	Consider $\mu_1  \leq \mu_2$. As $s \mapsto z(\mu)$ is monotone, we have
	\begin{align*}
		\Vert z(\mu_2) - z(\mu_1) \Vert_1 &= \langle \mathbbb{1}, z(\mu_2) - z(\mu_1) \rangle \\
		&= \langle \mathbbb{1}, \P_{\Span M} z(\mu_2) - \P_{\Span M} z(\mu_1) \rangle \\
		&\qquad+ \langle \mathbbb{1}, \P_{\ker M} z(\mu_2)\rangle -\langle  \mathbbb{1}, \P_{\ker M} z(\mu_1) \rangle \, . 
	\end{align*}
	Now fix a compact set $K$. As $\mu \in K \mapsto  \P_{\Span M} z(\mu)$ and $\mu \in K \mapsto \langle \mathbbb{1}, \P_{\ker M} z(\mu) \rangle$ are Lipschitz continuous, there exists a constant $C > 0$ such that for all $\mu_1, \mu_2 \in K$, $\mu_1 \leq \mu_2$,
	\begin{align*}
		\Vert z(\mu_2) - z(\mu_1) \Vert_1 \leq C |\mu_2 - \mu_1|\, .
	\end{align*}
	This shows that $\mu \geq 0 \mapsto z(\mu)$ is locally Lipschitz continuous and thus absolutely continuous on compact subsets of $[0,+\infty)$.
\end{proof}

\section{The \texorpdfstring{$u \circ v$}{uv} case --- proof of the results}
\label{sec:uv-proof}

The strategy of this section is to reduce the $u \circ v$ case to the $u^2$ case. The reduction methods are detailed in Sec.~\ref{sec:reductions}. In Sec.~\ref{sec:reduction-lasso}, we reduce the lasso problem to a positive lasso problem. In Sec.~\ref{sec:reduction-dynamics}, we reduce the DLN dynamics in the $u \circ v$ case to the dynamics in the $u^2$ case. Finally, in Sec.~\ref{sec:proof-uv}, we use these elements to deduce Thm.~\ref{thm:monotone-uv} from Thm.~\ref{thm:monotone-u2} and Thm.~\ref{thm:nonmonotone-uv} from Thm.~\ref{thm:nonmonotone-u2}.

\subsection{Reductions}
\label{sec:reductions}

\subsubsection{Reduction of the lasso to the positive lasso} 
\label{sec:reduction-lasso}

We study the minimization of the lasso objective:
\begin{align*}
	&\mindot_{x \in \R^d} \left\{\Lasso(x,\mu) = \ell(x)+\left(\lambda+\frac{1}{\mu}\right) \Vert x \Vert_1 \right\} \, , &&\ell(x) = \frac{1}{2} \langle x, Mx \rangle -\langle r, x \rangle \, . 
\end{align*}
We want to connect the lasso minimization problem to a positive lasso minimization problem~\eqref{eq:positive-lasso}. This can be seen as a complexification of the lasso minimization problem, as contrained optimization problems are arguably more delicate to deal with than unconstrained ones. However, the lasso objective is quadratic on nonnegative vectors, while it is not on the full space $\R^d$. The constrained minimization of a quadratic has some advantages over the unconstrained minimization of a non-smooth function. 

The strategy is to decompose $x$ as the difference $x = y^{\pos} - y^\neg$ of two nonnegative vectors $y^{\pos}, y^\neg \geq 0$. For instance, a canonical choice would be to take $y^{\pos} = x_+$ and $y^\neg = x_-$. Note that for such a decomposition, 
\begin{equation*}
	\ell(x) = \widetilde{\ell}(y) \, , \qquad y = \begin{pmatrix}
		y^{\pos} \\
		y^\neg
	\end{pmatrix} \, , 
\end{equation*} 
where $\widetilde{\ell}$ is the quadratic function
\begin{align*}
	&\widetilde{\ell}(y) = \frac{1}{2} \langle y, \widetilde{M} y \rangle - \langle \widetilde{r}, y \rangle \, , &&\widetilde{M} = \begin{pmatrix}
		M & -M \\
		-M & M
	\end{pmatrix} \, , &&\widetilde{r} = \begin{pmatrix}
		r \\
		-r
	\end{pmatrix} \, .
\end{align*}
We define the lasso objective associated to $\widetilde{\ell}$:
\begin{align*}
	\widetilde{\Lasso}(y,\mu) = \widetilde{\ell}(y) + \left(\lambda+\frac{1}{\mu}\right) \Vert y \Vert_1 \, . 
\end{align*}
We denote $\widetilde{\PosLasso}(\mu) = \min_{y\geq 0} \widetilde{\Lasso}(y, \mu)$ the minimum of the positive lasso optimization problem associated to $\widetilde{\ell}$.

\begin{lem}
	\label{lem:lasso-reduction}
\begin{enumerate}[label=(\arabic*)]
	\item\label{it:ineq-lasso} For any $y = (y^\pos, y^\neg) \in \R^{2d}_{\geq 0}$, 
	\begin{equation*}
		\Lasso(y^\pos - y^\neg, \mu) \leq \widetilde{\Lasso}(y, \mu) \, , 
	\end{equation*}
	with equality if and only if $\langle y^\pos, y^\neg \rangle = 0$. 
	\item\label{it:minimizer-decomposition} If $z$ minimizes $\Lasso(., \mu)$ over $\R^d$, then $y = (z_+, z_-) \in \R^{2d}_{\geq 0}$ minimizes $\widetilde{\Lasso}(., \mu)$ over $\R^{2d}_{\geq 0}$.
	\item\label{it:minimum-equality} $\Lasso_*(\mu) = \widetilde{\PosLasso}_*(\mu)$. 
\end{enumerate}
\end{lem}

\begin{proof}
	\begin{enumerate}[label=(\arabic*)]
		\item Consider $y = (y^\pos, y^\neg) \in \R^{2d}_{\geq 0}$. 
		\begin{align*}
			\Lasso(y^\pos - y^\neg, \mu) &= \ell(y^\pos - y^\neg) + \left(\lambda+\frac{1}{\mu}\right) \Vert y^\pos - y^\neg \Vert_1 \\
			&\leq \ell(y^\pos - y^\neg) + \left(\lambda+\frac{1}{\mu}\right) \left(\Vert y^\pos \Vert_1 + \Vert y^\neg \Vert_1\right) \\
			&= \widetilde{\ell}(y) + \left(\lambda+\frac{1}{\mu}\right) \Vert y \Vert_1 \\
			&= \widetilde{\Lasso}(y, \mu) \, .
		\end{align*}
		There is equality if and only if there is equality in the triangle inequality, which happens if and only if $\langle y^\pos, y^\neg \rangle = 0$. 
		\item Let $z$ be a minimizer of $\Lasso(., \mu)$ over $\R^d$. Consider $y = (y^\pos, y^\neg) \in \R^{2d}_{\geq 0}$. Then, from \ref{it:ineq-lasso}, 
		\begin{align*}
			\widetilde{\Lasso}(y, \mu) &\geq \Lasso(y^\pos - y^\neg, \mu) \geq \Lasso(z, \mu)  = \widetilde{\Lasso}((z_+, z_-), \mu) \, ,
		\end{align*}
		where the equality holds in the last step as $\langle z_+, z_- \rangle = 0$. As a consequence, $(z_+, z_-)$ minimizes $\widetilde{\Lasso}(., \mu)$ over $\R^{2d}_{\geq 0}$.
		\item From \ref{it:ineq-lasso}, we have $\Lasso_*(\mu) \leq \widetilde{\PosLasso}(\mu)$. Further, if $z$ is a minimizer of $\Lasso(., \mu)$ over $\R^d$, then 
		\begin{align*}
			\Lasso_*(\mu) &= \Lasso(z, \mu) = \widetilde{\Lasso}((z_+, z_-), \mu) \geq \widetilde{\PosLasso}(\mu) \, . 
		\end{align*}
		This proves the equality.
	\end{enumerate}

\end{proof}

\subsubsection{Reduction of dynamics in the \texorpdfstring{$u \circ v$}{uv} case to the \texorpdfstring{$u \circ u$}{u squared} case}
\label{sec:reduction-dynamics}

Recall that the dynamics in the $u\circ v$ case are defined by
\begin{align*}
	&\frac{\diff u}{\diff t} = - \nabla_u L(u,v) = -v \circ \nabla \ell(u\circ v) - \lambda u  \, , \\
	&\frac{\diff v}{\diff t} = - \nabla_v L(u,v) = -u \circ \nabla \ell(u\circ v) - \lambda v  \, .
\end{align*}
Consider 
\begin{align*}
	&p^\pos = \frac{1}{2}(u+v) \, , &&p^\neg = \frac{1}{2}(u-v) \, .
\end{align*}
Then 
\begin{align*}
	&u\circ v = (p^\pos)^2 - (p^\neg)^2 \, .
\end{align*}
Define $\widetilde{t} = \frac{t}{2}$. Then 
\begin{align*}
	\frac{\diff p^\pos}{\diff \widetilde{t}} &= \frac{\diff t}{\diff \widetilde{t}} \frac{1}{2} \left(\frac{\diff u}{\diff t} + \frac{\diff v}{\diff t}\right) = 
	\frac{\diff u}{\diff t} + \frac{\diff v}{\diff t} \\
	&= - v \circ \nabla \ell(u\circ v) - u \circ \nabla \ell(u\circ v) - \lambda (u+v) \\
	&= - 2 p^\pos \circ \nabla \ell((p^\pos)^2 - (p^\neg)^2) - 2 \lambda p^\pos \\
	&= - 2 p^\pos \circ \nabla_{y^\pos} \widetilde{\ell} ((p^\pos)^2, (p^\neg)^2) - 2 \lambda p^\pos \\
\end{align*}
Defining 
\begin{equation}
	\label{eq:def-L-tilde-ap}
	\widetilde{L}(p^\pos, p^\neg) = \widetilde{\ell}((p^\pos)^2, (p^\neg)^2)  + \lambda \left(\Vert p^\pos \Vert^2 + \Vert p^\neg \Vert^2 \right) \, ,
\end{equation}
the above equation can be written as
\begin{equation*}
	\frac{\diff p^\pos}{\diff \widetilde{t}} = - \nabla_{p^\pos} \widetilde{L}(p^\pos, p^\neg) \, . 
\end{equation*}
A similar computation gives that, 
\begin{align*}
	\frac{\diff p^\neg}{\diff \widetilde{t}} 
	&= - \nabla_{p^\neg} \widetilde{L}(p^\pos, p^\neg) \, . 
\end{align*}
Thus, 
\begin{equation}
	\label{eq:red}
	\frac{\diff }{\diff \widetilde{t} } (p^\pos, p^\neg) = - \nabla_{(p^\pos, p^\neg)} \widetilde{L} (p^\pos, p^\neg)  \, .
\end{equation}
Compare Eqs.~\eqref{eq:def-L-tilde-ap}, \eqref{eq:red} with Eqs.~\eqref{eq:def-L-u2}, \eqref{eq:flow-u2-new}. 
We have reduced a gradient flow on $L$ in the $u \circ v$ case to a gradient flow on $\widetilde{L}$ in the $u \circ u$ case.

\subsection{Proofs}
\label{sec:proof-uv}

The proofs of Thms.~\ref{thm:monotone-uv} and \ref{thm:nonmonotone-uv} are provided by reductions to Thms.~\ref{thm:monotone-u2} and \ref{thm:nonmonotone-u2} respectively. 

In each theorem, we are provided a minimizer ${x}(\mu)$ of $\Lasso\left(., \mu\right)$ over $\R^d$. By Lemma~\ref{lem:lasso-reduction}\ref{it:minimizer-decomposition}, $({x}_+(\mu), {x}_-(\mu))$ minimizes $\widetilde{\Lasso}\left(., \mu\right)$ over $\R^{2d}_{\geq 0}$. This minimizer of the positive lasso will be used in applying Thms~\ref{thm:monotone-u2} and \ref{thm:nonmonotone-u2}. 

Further, we use the reduction of Sec.~\ref{sec:reduction-dynamics} from the $u \circ v$ case to the $u^2$ case. We keep the same notations in this section, but now denote the dependence in $\varepsilon$ of the solutions. Thus, with $\widetilde{t} = \frac{t}{2}$, we have 
\begin{equation*}
	\frac{\diff }{\diff \widetilde{t} } (
		p^{\pos,\varepsilon},
		p^{\neg,\varepsilon})
	 = - \nabla_{(p^{\pos,\varepsilon}, p^{\neg, \varepsilon})} \widetilde{L} (
		p^{\pos,\varepsilon} ,
		p^{\neg,\varepsilon})
	  \, .
\end{equation*}
Again, this is an equation of the form \eqref{eq:flow-u2-new} in the $u^2$ case. 

Denote $(y^{\pos,\varepsilon}, y^{\neg,\varepsilon}) = ((p^{\pos,\varepsilon})^2, (p^{\neg,\varepsilon})^2)$ and $(\overline{y}^{\pos,\varepsilon}, \overline{y}^{\neg,\varepsilon})$ its time-average. The application of Thms.~\ref{thm:monotone-u2} and \ref{thm:nonmonotone-u2} decribe the performance of $(\overline{y}^{\pos,\varepsilon}, \overline{y}^{\neg,\varepsilon})$ in the minimization of the positive lasso objective $\widetilde{\Lasso}\left(., {\widetilde{s}}\right)$, where $\widetilde{s}$ is defined in coherence with Eq.~\eqref{eq:rescaled-time-u2}:
\begin{equation*}
	\widetilde{s} = \frac{4}{\log \frac{1}{\varepsilon}} \widetilde{t} = \frac{2}{\log \frac{1}{\varepsilon}} {t} \underset{\text{(Eq.~\eqref{eq:rescaled-time-uv})}}{=} s \, . 
\end{equation*}
Note that there is a correspondence between the initializations: 
\begin{equation*}
	\alpha = \left(\frac{1}{2}(\beta+\gamma),\frac{1}{2}(\beta-\gamma)\right) \, . 
\end{equation*}
As a consequence, the assumption that $\beta_i \neq \pm \gamma_i$, $i = 1, \ldots, d$, is equivalent to $\alpha_i \neq 0$, $i = 1, \ldots, 2d$.

\subsubsection{Proof of Theorem~\ref{thm:monotone-uv}}

In Thm.~\ref{thm:monotone-uv}, it is assumed that $\mu \mapsto \mu{x}(\mu)$ is coordinate-wise monotone. As a consequence, $\mu \mapsto \mu({x}_+(\mu), {x}_-(\mu))$ is coordinate-wise non-decreasing. Thus we can apply Thm.~\ref{thm:monotone-u2}. We obtain 
\begin{equation*}
	\widetilde{\Lasso}\left((\overline{y}^{\pos,\varepsilon}, \overline{y}^{\neg,\varepsilon}), {s}\right) \xrightarrow[\varepsilon \to 0]{} \widetilde{\PosLasso}_*\left({s}\right) \, .
\end{equation*}
We recall from Sec.~\ref{sec:reduction-dynamics} that $x^\varepsilon = u^\varepsilon \circ v^\varepsilon = (p^{\pos, \varepsilon})^2 - (p^{\neg, \varepsilon})^2 = y^{\pos, \varepsilon} - y^{\neg, \varepsilon}$. Thus, using Lemma \ref{lem:lasso-reduction}\ref{it:ineq-lasso}, \ref{it:minimum-equality}, we obtain 
\begin{align*}
	\Lasso\left(\overline{x}^\varepsilon(s), {s}\right) - \Lasso_*\left({s}\right) &= \Lasso\left(\overline{y}^{\pos,\varepsilon}(s)- \overline{y}^{\neg,\varepsilon}(s), {s}\right) - \Lasso_*\left({s}\right) \\
	&\leq \widetilde{\Lasso}\left((\overline{y}^{\pos,\varepsilon}, \overline{y}^{\neg,\varepsilon}), {s}\right) - \widetilde{\PosLasso}_*\left({s}\right) \\
	&\xrightarrow[\varepsilon \to 0]{} 0 \, . 
\end{align*}

\subsubsection{Proof of Theorem~\ref{thm:nonmonotone-uv}}

In Thm.~\ref{thm:nonmonotone-uv}, it is assumed that $\mu \mapsto \mu{x}(\mu)$ is absolutely continuous. As a consequence, $\mu \mapsto \mu({x}_+(\mu), {x}_-(\mu))$ is absolutely continuous. Thus we can apply Thm.~\ref{thm:nonmonotone-u2}. 

Note that the quantity $z^\downarrow(\mu)$ defined in Eq.~\eqref{eq:z-downarrow-u2}, when ${x}(\mu)$ is replaced by $({x}_+(\mu), {x}_-(\mu))$, corresponds to the quantity $z^\downarrow(\mu)$ defined in Eq.~\eqref{eq:z-downarrow-uv}. Thus, applying Thm.~\ref{thm:nonmonotone-u2}, we obtain 
\begin{align*}
	&\limsup_{\varepsilon \to 0} \widetilde{\Lasso} \left((\overline{y}^{\pos,\varepsilon}, \overline{y}^{\neg,\varepsilon}) , {s}\right)  \\
	&\qquad\leq \widetilde{\PosLasso}_*\left({s}\right) + C \eta(\lambda,s,z^\downarrow(s)) \, . 
\end{align*}
Thus, combining with Lemma \ref{lem:lasso-reduction}\ref{it:ineq-lasso}, \ref{it:minimum-equality}, we obtain 
\begin{align*}
	\limsup_{\varepsilon \to 0} {\Lasso} \left(\overline{\var}^\varepsilon(s) , {s}\right) &= \limsup_{\varepsilon \to 0} {\Lasso} \left(\overline{y}^{\pos,\varepsilon}(s)- \overline{y}^{\neg,\varepsilon}(s) , {s}\right) \\
	&\leq \limsup_{\varepsilon \to 0} \widetilde{\Lasso} \left((\overline{y}^{\pos,\varepsilon}, \overline{y}^{\neg,\varepsilon}) , {s}\right) \\
	&\leq \widetilde{\PosLasso}_*\left({s}\right) + C \eta(\lambda,s,z^\downarrow(s)) \\
	&\leq {\Lasso}_*\left({s}\right) + C \eta(\lambda,s,z^\downarrow(s)) \, .
\end{align*}

\section*{Acknowledgements}

The author is grateful to Loucas Pillaud-Vivien for numerous stimulating discussions and for providing valuable insights that have improved the development of this work, and to Florent Krzakala for helpful comments. The author also acknowledges support from the ANR and the Ministère de l’Enseignement Supérieur et de la Recherche.

\addcontentsline{toc}{section}{References}
\bibliographystyle{plain}
\bibliography{bibliography}

\end{document}